\documentclass[11pt]{article}
\usepackage{fullpage}
\usepackage{amsmath,amssymb,amsthm,bbm}
\usepackage{algorithm,algpseudocode}
\usepackage{graphicx,subfigure}
\usepackage{threeparttable}
\usepackage[inline]{enumitem}
\usepackage{multirow}
\usepackage{url}
\usepackage{bm}
\usepackage{color}
\usepackage[table]{xcolor}
\usepackage{mathtools}
\usepackage{leftidx}
\usepackage{pdflscape}
\usepackage{afterpage}
\usepackage{pbox}
\usepackage{pdfpages}

\graphicspath{{figs/}{figs2/}}

\newcommand{\Romannumber}[1]{\uppercase\expandafter{\romannumeral #1}}

\newcommand{\complex}{\mathbb{C}}
\newcommand{\real}{\mathbb{R}}

\newcommand{\im}{\mathbf{i\,}}

\DeclareMathOperator{\erf}{erf}
\DeclareMathOperator{\erfc}{erfc}
\DeclareMathOperator{\erfi}{erfi}

\DeclareMathOperator{\var}{Var}

\DeclareMathOperator{\sinc}{sinc}

\theoremstyle{definition} 
\theoremstyle{remark}     \newtheorem{remark}{Remark}
\theoremstyle{remark}     
\theoremstyle{plain}      \newtheorem{theorem}{Theorem}
\theoremstyle{plain}      
\theoremstyle{plain}      
\theoremstyle{plain}      \newtheorem{corollary}[theorem]{Corollary}
\theoremstyle{plain}

        {\vskip 10pt \begin{algorithmic}[1]}%
        {\end{algorithmic} \vskip 10pt}

\begin{document}

\title{\textsf{On Bochner's and Polya's Characterizations of Positive-Definite Kernels and the Respective Random Feature Maps}}
\author{Jie Chen\thanks{IBM Thomas J. Watson Research Center. Email: \texttt{chenjie@us.ibm.com}}
  \and Dehua Cheng\thanks{University of Southern California. Email: \texttt{(dehua.cheng, yanliu.cs)@usc.edu}}
  \and Yan Liu\footnotemark[2]
}
\maketitle

\begin{abstract}
Positive-definite kernel functions are fundamental elements of kernel methods and Gaussian processes. A well-known construction of such functions comes from Bochner's characterization, which connects a positive-definite function with a probability distribution. Another construction, which appears to have attracted less attention, is Polya's criterion that characterizes a subset of these functions. In this paper, we study the latter characterization and derive a number of novel kernels little known previously.

In the context of large-scale kernel machines, Rahimi and Recht (2007) proposed a random feature map (random Fourier) that approximates a kernel function, through independent sampling of the probability distribution in Bochner's characterization. The authors also suggested another feature map (random binning), which, although not explicitly stated, comes from Polya's characterization. We show that with the same number of random samples, the random binning map results in an Euclidean inner product closer to the kernel than does the random Fourier map. The superiority of the random binning map is confirmed empirically through regressions and classifications in the reproducing kernel Hilbert space.
\end{abstract}

\section{Introduction}
A positive-definite function (also coined \emph{kernel} in this paper) is a complex-valued function $k:\real\to\complex$ such that for any $n$ real numbers $x_1,x_2,\ldots,x_n$, the matrix $K$ with elements $K_{ij}=k(x_i-x_j)$ is positive semi-definite. A well-known relationship between positive-definite functions and probability distributions is given by the celebrated Bochner's theorem~\cite{Stein1999}, which states that a continuous function $k$ with $k(0)=1$ is positive-definite if and only if it is the characteristic function (cf) of some random variable $X$. Let $F(x)$ be the cumulative distribution function (cdf) of $X$; then, the if-and-only-if condition is written in the form of a Stieltjes integral as
\begin{equation}\label{eqn:rff}
k(r)=E[e^{\im rX}]=\int_{\real} e^{\im rx}\,dF(x).
\end{equation}
A practical significance of this characterization is that one may construct a positive-definite function from a probability distribution. For example, the squared exponential kernel\footnote{Also called the Gaussian kernel.} is constructed from the normal distribution, the exponential kernel from the Cauchy distribution\footnote{It turns out that the exponential kernel can be generalized to high dimensions by taking either the 1-norm or the 2-norm of the inputs. In the 1-norm case, the kernel is also called the Laplace kernel and the distribution is a tensor product of one-dimensional Cauchy distributions. In the 2-norm case, the distribution is multivariate Cauchy. More discussions appear in Section~\ref{sec:discuss}.}, and the Cauchy kernel from the Laplace distribution. Positive-definite functions are of vital importance in kernel methods and Gaussian processes. The kernel $k$ in kernel methods defines a reproducing kernel Hilbert space (RKHS)~\cite{Aronszajn1950}, from which an optimal prediction function is sought with respect to some risk functional~\cite{Schoelkopf2001a,Hastie2009}. In Gaussian processes, $k$ serves as a covariance function and its Fourier transform, coined \emph{spectral density}, dictates the smoothness and other behavior of the process~\cite{Stein1999,Wendl2004,Rasmussen2006,Chiles2012}.

Another approach for constructing kernels from probability distributions, which appears to have attracted less attention, comes from Polya's criterion~\cite{Durrett2010}, which states that for any real continuous and even function $k$ convex on $[0,\infty)$ with $k(0)=1$ and $k(\infty)=0$, there exists a random variable $X$ with a positive support and a cdf $F(x)$ such that
\begin{equation}\label{eqn:k}
k(r)=\int_{\real}\max\left\{0,1-\frac{|r|}{|x|}\right\}\,dF(x).
\end{equation}
An informal argument why $k$ is positive-definite, is that the integrand in~\eqref{eqn:k} is the triangular function, whose Fourier transform is the squared sinc function that is nonnegative. With slightly extra work, in Section~\ref{sec:construction}, we show that the converse of Polya's criterion is also true; that is, given any cdf $F$ with a positive support, the function $k$ defined in~\eqref{eqn:k} possesses the said properties. Hence, \eqref{eqn:k} in fact characterizes a subset of positive-definite functions, the most salient property being convexly decreasing on $[0,\infty)$; and the respective probability distributions are those positively supported. We study in depth Polya's criterion and its consequences, particularly in the context of kernel constructions, in Section~\ref{sec:construction}. Then, in Section~\ref{sec:dist}, we consider a number of example distributions and derive explicit expressions for $k$ and the associated Fourier transform. Such distributions include Poisson, gamma, Nakagami, Weibull, and other distributions that are special cases of the last three (e.g., exponential, chi-square, chi, half-normal, and Rayleigh).

One may recall that~\eqref{eqn:k} resembles an equality established by Rahimi and Recht~\cite{Rahimi2007}
\begin{equation}\label{eqn:rbf}
k(r)=\int_0^{\infty}\max\left\{0,1-\frac{r}{x}\right\}x k''(x)\,dx, \qquad r\ge0,
\end{equation}
for any twice differentiable function $k$ on $(0,\infty)$ that vanishes at the infinity. Indeed, if $x k''(x)$ integrates to unity, it could be considered the probability density function (pdf) associated to $F$. Two important distinctions, however, should be noted. First, the expression~\eqref{eqn:rbf} implicitly assumes the existence of a pdf, which occurs only for (well behaved) continuous distributions. To the contrary,~\eqref{eqn:k}, in the form of a Stieltjes integral, is more general, defined for distributions including notably discrete ones. Such does not contradict with~\eqref{eqn:rbf}, because a kernel function constructed from a discrete distribution may not be twice differentiable on $(0,\infty)$; in fact, it is at most once differentiable.

Second, \eqref{eqn:k} and~\eqref{eqn:rbf} result in methods that utilize the relationship between a kernel function and a probability distribution in a completely opposite direction. The work~\cite{Rahimi2007}, based on~\eqref{eqn:rbf}, starts from a known kernel $k$ and seeks a valid pdf for constructing random feature maps. On the other hand, our work, based on~\eqref{eqn:k}, focuses on constructing new kernels. The theoretical appeal of~\eqref{eqn:k} guarantees that the so defined function $k$ is always a valid kernel; and the prosperous results in probability distributions provide a practical opportunity to derive explicit formulas for novel kernels $k$ little known previously.

Whereas the mathematical properties of the proposed kernels are interesting in their own right, the work here stems from a practical purpose: we are interested in comparing the quality of the random feature maps resulting from~\eqref{eqn:rff} and~\eqref{eqn:k}, if possible, for the same kernel function $k$. A computational bottleneck in many kernel and Gaussian process applications is the factorization of the $n\times n$ shifted kernel matrix $K+\lambda I$, whose memory cost and time cost scale as $O(n^2)$ and $O(n^3)$, respectively, if no particular structures of $K$ are known other than symmetry. A class of methods aiming at reducing the computational costs is the random feature approaches~\cite{Rahimi2007,Kar2012,Vedaldi2012,Le2013,Yang2014}, which map a data point $x$ to a random vector $\bm{z}(x)$ such that $\bm{z}(x)^T\bm{z}(x')$ approximates $k(x-x')$ for any pair of points $x$ and $x'$. In the matrix form, let $\bm{z}$ be a column vector and let $Z$ be the matrix $[\bm{z}(x_1),\bm{z}(x_2),\ldots,\bm{z}(x_n)]$; then, $Z^TZ$ approximates $K$.

Probably the most well-known and used random feature map is random Fourier (see, e.g., the original publication~\cite{Rahimi2007}, a few extensions~\cite{Yen2014,Dai2014,Sindhwani2015}, analysis~\cite{Yang2012,Wu2016}, and applications~\cite{Huang2014,Chen2016}); whereas a less popular, but more effective one as we argue in this paper, is random binning (see the same publication~\cite{Rahimi2007}). The random Fourier approach uses~\eqref{eqn:rff} to construct a dense $Z$ of size $D\times n$, where $D$ denotes the number of random Fourier samples. The random binning approach, as we extend in this work for an arbitrary distribution positively supported, uses~\eqref{eqn:k} to construct a sparse $Z$ where each column has $D'$ nonzeros, with $D'$ denoting the number of random binning samples. We analyze in Section~\ref{sec:var} that $Z^TZ$ better approximates $K$ by using the latter approach, if $D$ is taken to be the same as $D'$. In other words, for a matching approximation quality, $D'$ may be (much) smaller than $D$. Such an observation supports the use of the proposed formula~\eqref{eqn:k} for kernel construction and approximation.

Note that analysis of the two random feature approaches exists in other works. Rahimi and Recht~\cite{Rahimi2007} give probabilistic bounds for the uniform error $\sup_{x,x'}|\bm{z}(x)^T\bm{z}(x')-k(x-x')|$. These results, however, do not directly compare the two approaches as we do. Wu et.~al~\cite{Wu2016} consider the setting of risk functional minimization in the RKHS and bound the bias of the computed function from the optimal one, when the minimization is done through coordinate descent by taking one sample at a time. They argue that the optimization converges faster for the random binning approach, in the sense that if the same number of samples/iterations are used, the bias has a smaller upper bound. On the other hand, our analysis focuses on the matrix approximation error and gives exact values rather than bounds. As a result, the analysis also favors the random binning approach. Experimental results that gauge regression and classification performance further confirm the superiority of this approach; see Section~\ref{sec:exp}.

We summarize the contributions of this work and conclude in Section~\ref{sec:conclude}.

\section{Polya's Characterization}\label{sec:construction}
We start with the formal statement of Polya.

\begin{theorem}[Polya's criterion]\label{thm:polya}
If $k:\real\to[0,1]$ is a real, continuous and even function with $k(0)=1$, $\lim_{r\to\infty}k(r)=0$, and $k$ is convex on $[0,\infty)$, then there exists a cumulative distribution function $F(x)$ on $(0,\infty)$ such that
\begin{equation}\label{eqn:k2}
k(r)=\int_0^{\infty}\max\left\{0,1-\frac{|r|}{|x|}\right\}\,dF(x).
\end{equation}
Hence, $k$ is a characteristic function.
\end{theorem}

\begin{proof}
See, e.g., proof of Theorem 3.3.10 in~\cite{Durrett2010}.
\end{proof}

Polya's criterion sheds deep insights between a kernel $k$ and a cdf $F$ connected by the relation~\eqref{eqn:k2}. Let us stress a few.

First, being a characteristic function is equivalent to being positive-definite with $k(0)=1$, a consequence of Bochner's theorem. The positive definiteness comes from the fact that the integrand in~\eqref{eqn:k2} is a triangular function with scaled width $|x|$. Based on the well-known relation
\[
\int_{\real}e^{\im rt}\max\left\{0,1-|r|\right\}\,dr=\frac{4}{t^2}\sin\left(\frac{t}{2}\right)^2,
\]
if $k$ is absolutely integrable, then $k$ admits an inverse Fourier transform
\begin{equation}\label{eqn:ft}
\frac{1}{2\pi}\int_{-\infty}^{\infty}k(r)e^{-\im rt}\,dr
=\frac{1}{2\pi}\int_0^{\infty}x\sinc^2\left(\frac{xt}{2}\right)\,dF(x)=:h(t),
\end{equation}
where $\sinc(x)=\sin(x)/x$. Clearly, $h$ is nonnegative for all $t$. Then, the Bochner's characterization~\eqref{eqn:rff} is satisfied with a stronger condition for the cdf, one that admits a density:
\[
k(r)=\int_{\real}e^{\im rx}h(x)\,dx.
\]
If $k$ is not absolutely integrable, one invokes L\'{e}vy's continuity theorem and shows that a sequence of absolutely integrable and positive-definite functions converge to $k$. Both cases conclude that $k$ is positive-definite.

Second, the cdf $F$ in~\eqref{eqn:k2} may be constructed as
\[
F(x)=1-k(x)+xg(x)
\qquad\text{with}\qquad
g(x)=\lim_{\delta\to0^+}\frac{k(x+\delta)-k(x)}{\delta}.
\]
Here, $g$ is the right derivative of $k$. By the continuity and convexity of $k$, $g$ is well defined. Clearly, if $k$ is differentiable, then $F(x)=1-k(x)+xk'(x)$. Further, if $k$ is twice differentiable, we have that $F$ is differentiable with
\[
F'(x)=xk''(x),
\]
which recovers~\eqref{eqn:rbf} established in~\cite{Rahimi2007}.

Third, it is important to note that $F$ is supported on $(0,\infty)$, not $[0,\infty)$. If one considers the domain of a cdf to be the whole real line, then the requirement for $F$ in the theorem may be equivalently stated as $F(x)=0$ for all $x\le0$. Apart from an obvious practical constraint seen later, that the random variable $X$ will be used as the width of a bin, which must be positive, we particularly note that $F(0)$ cannot be nonzero. Such a constraint is naturally satisfied by continuous distributions, because $F$ must be continuous at $0$. However, a discrete distribution may assign a nonzero mass for $X=0$, which makes $F(0)\ne0$, a case we must rule out in the theorem. The reason is that if $\Pr(X=0)$ is nonzero, then $dF(x)$ makes a nontrivial contribution to the Stieltjes integral~\eqref{eqn:k2} when $x$ approaches $0$ from the right. In such a case, $k(r)$ does not converge to $1$ when $r\to0$. In other words, we have to sacrifice either the equality $k(0)=1$ or the continuity of $k$ in the theorem, if we want to relax the support of $F$ to $[0,\infty)$ with particularly allowing $F(0)\ne0$. This is not a sacrifice we make in this paper. Later in Section~\ref{sec:dist} when we construct kernels from discrete distributions, we will reiterate the requirement that $\Pr(X=0)=0$.

It is not hard to show that the converse of Theorem~\ref{thm:polya} is also true. Then, we strengthen the theorem into the following result and call it \emph{Polya's characterization}. The significance is that any distribution on $(0,\infty)$ defines a positive-definite function, an additional characterization besides that of Bochner's.

\begin{corollary}\label{cor:polya}
A real function $k$ is continuous and even with $k(0)=1$, $\lim_{r\to\infty}k(r)=0$, and convex on $[0,\infty)$, if and only if there exists a cdf $F(x)$ with positive support such that
\[
k(r)=\int_0^{\infty}\max\left\{0,1-\frac{|r|}{|x|}\right\}\,dF(x).
\]
Moreover, all such functions $k$ are positive-definite.
\end{corollary}

\begin{proof}
Theorem~\ref{thm:polya} corresponds to the ``only if'' part. Hence, it suffices to show the ``if ''part. Clearly, $k$ is continuous, even, and satisfies $k(0)=1$ and $\lim_{r\to\infty}k(r)=0$. We therefore focus on only the convexity.

Let $r_1\ge0$, $r_2\ge0$, $r_1\ne r_2$, and $t\in[0,1]$. Define
\[
L=\max\left\{0,\,\,t\left[1-\frac{r_1}{|x|}\right]+(1-t)\left[1-\frac{r_2}{|x|}\right]\right\}
\]
and
\[
R=\max\left\{0,\,\,t\left[1-\frac{r_1}{|x|}\right]\right\}
+\max\left\{0,\,\,(1-t)\left[1-\frac{r_2}{|x|}\right]\right\}.
\]
When $r_1$ and $r_2$ are on the same side of $|x|$, we have $L=R$. When $r_1\le|x|\le r_2$, we have
\[
L\le\max\left\{0,\,\,t\left[1-\frac{r_1}{|x|}\right]\right\}\le R.
\]
Similarly, when $r_2\le|x|\le r_1$, we have
\[
L\le\max\left\{0,\,\,(1-t)\left[1-\frac{r_2}{|x|}\right]\right\}\le R.
\]
Hence, all cases point to that $L\le R$. Therefore,
\[
k(tr_1+(1-t)r_2)\le tk(r_1)+(1-t)k(r_2),
\]
concluding the convexity of $k$ on $[0,\infty)$.
\end{proof}

Because the central subject of this paper, the function $k$ in Corollary~\ref{cor:polya}, is even, its Fourier transform and inverse transform differ by only a factor of $2\pi$. In what follows, we do not distinguish the two transforms and consider only the forward one, with formal notation
\[
\mathcal{F}[k](t)\equiv\int_{-\infty}^{\infty}k(r)e^{\im rt}\,dr.
\]
We will also use Fourier transforms in the more general setting---one that is defined for generalized functions---which does not require $k$ to be absolutely integrable.

\subsection{Special Case}\label{sec:special.case}
Based on the foregoing, because $F(x)=0$ for all $x\le0$, we may get rid of the max operator and write equivalently,
\begin{equation}\label{eqn:k.alt}
k(r)=\int_r^{\infty}\left(1-\frac{r}{x}\right)\,dF(x)
=\int_r^{\infty}dF(x)
-r\int_r^{\infty}\frac{dF(x)}{x}, \qquad r\ge0,
\end{equation}
omitting the obvious symmetric part $r<0$. The second term on the right-hand side of~\eqref{eqn:k.alt}, $r\int_r^{\infty}(1/x)dF(x)$, is finite when $r\to0^+$, but not necessarily when the front factor $r$ is dropped. In this subsection, we consider the special, but not-so-infrequent case, when
\[
\int_r^{\infty}\frac{dF(x)}{x}
\]
indeed converges to a finite number as $r\to0^+$. A benefit of considering this case is that we may introduce another random variable to simplify the expressions for $k$ and its Fourier transform sometimes. Later in Section~\ref{sec:dist} we show quite a few such examples. For convenience, the integration limit starts from $-\infty$ rather than $0$.

Formally, let $X$ be a random variable with cdf $F(x)$, where $F(x)=0$ for all $x\le0$. If
\begin{equation}\label{eqn:C}
C:=\int_{-\infty}^{\infty}\frac{dF(x)}{x}
\end{equation}
is finite, define
\begin{equation}\label{eqn:tilde.F}
\widetilde{F}(x):=\int_{-\infty}^x\frac{dF(t)}{Ct}.
\end{equation}
Because $\widetilde{F}(-\infty)=0$, $\widetilde{F}(\infty)=1$, and $\widetilde{F}$ is nondecreasing and right continuous, it is the cdf of some random variable $\widetilde{X}$. The following theorem gives the expressions of $k$ and its Fourier transform by using some quantities with respect to $X$ and $\widetilde{X}$. For notational consistency, we will use $F_{\widetilde{X}}$ to replace $\widetilde{F}$ when appropriate.

\begin{theorem}\label{thm:special.case}
Denote by $F_Z$ and $\varphi_Z$ the cdf and the cf of a random variable $Z$, respectively. If $C$ defined in~\eqref{eqn:C} is finite and $\widetilde{X}$ is the respective random variable of $\widetilde{F}$ defined in~\eqref{eqn:tilde.F}, then,
\[
k(r)=[1-F_X(r)]-Cr[1-F_{\widetilde{X}}(r)], \qquad r\ge0,
\]
and
\[
\mathcal{F}[k](t)
=\frac{C}{t^2}[2-\varphi_{\widetilde{X}}(t)-\varphi_{\widetilde{X}}(-t)].
\]
\end{theorem}

\begin{proof}
The expression of $k$ is straightforward in light of~\eqref{eqn:k.alt}. To show the Fourier transform, we apply~\eqref{eqn:ft} and write
\[
\int_{-\infty}^{\infty}k(r)e^{\im rt}\,dr=
\frac{1}{t^2}\int_{-\infty}^{\infty}\frac{2-2\cos(xt)}{x}\,dF(x)=
\frac{C}{t^2}\int_{-\infty}^{\infty}[2-2\cos(xt)]\,d\widetilde{F}(x).
\]
Then, we have
\[
\frac{C}{t^2}\int_{-\infty}^{\infty}[2-2\cos(xt)]\,d\widetilde{F}(x)
=\frac{C}{t^2}\left[\int_{-\infty}^{\infty}2\,d\widetilde{F}(x)
-\int_{-\infty}^{\infty}e^{\im xt}\,d\widetilde{F}(x)
-\int_{-\infty}^{\infty}e^{-\im xt}\,d\widetilde{F}(x)\right],
\]
which simplifies to the second equality in the theorem.
\end{proof}

This theorem is extensively applied in Section~\ref{sec:dist}. Let us note two cases. For the case of discrete distributions, let $S$ be the support and denote by $f$ the probability mass function (pmf). Then, \eqref{eqn:C} and~\eqref{eqn:tilde.F} read
\begin{equation}\label{eqn:C1}
C=\sum_{x\in S}\frac{f(x)}{x}
\quad\text{and}\quad
\widetilde{f}(x)=\frac{f(x)}{Cx},
\end{equation}
where $\widetilde{f}$ is the pmf of the new random variable $\widetilde{X}$ stated in the theorem. In particular, if the elements of $S$ are all $\ge1$, or if the number of elements $<1$ is finite, or if the number of elements $<1$ is infinite but all are bounded away from $0$, then $C$ must be finite.

For the case of continuous distributions, if $F$ is differentiable on $(0,\infty)$ and $f$ is the corresponding pdf (i.e., $f=F'$), then~\eqref{eqn:C} and~\eqref{eqn:tilde.F} become
\begin{equation}\label{eqn:C2}
C=\int_0^{\infty}\frac{f(x)}{x}\,dx
\quad\text{and}\quad
\widetilde{f}(x)=\frac{f(x)}{Cx},
\end{equation}
where $\widetilde{f}$ is the pdf of the new random variable $\widetilde{X}$ stated in the theorem.

\subsection{Scaling}\label{sec:scaling}
Substantial experiences in kernel methods suggest that the spread of a kernel is one of the most important factors that affect the performance of a regression/classification. A well-known (though improper) example is the scale parameter $\sigma$ in a squared exponential kernel $k(r)=\exp[-r^2/(2\sigma^2)]$. This example is improper because the kernel does not correspond to any cdf $F$ in~\eqref{eqn:k}; nevertheless, the spirit of the example is that one needs to properly scale a kernel in order to achieve optimal results.

Hence, we introduce a scaling factor $\rho>0$ and turn $k(r)$ to $k(\rho r)$. Because of the vast difference in spreads among kernels constructed from different cdf's, a principled approach is to define $\rho=A/\tau$, where $A$ is used to standardize all kernels and $\tau$ is a tuning parameter that adjust the spread of the standardized kernel. One approach of standardization is to let $A$ be the area under curve, because then the area under $k(Ar)$ is $1$. The following result gives $A$.

\begin{theorem}\label{thm:area.under.curve}
If the random variable $X$ has a finite mean, then for $k$ defined in Corollary~\ref{cor:polya} we have
\begin{equation}\label{eqn:A}
\int_{-\infty}^{\infty}k(r)\,dr=E[X].
\end{equation}
\end{theorem}

\begin{proof}
Because $k$ is even, a direct calculation gives
\begin{align*}
\int_{-\infty}^{\infty}k(r)\,dr
&=2\int_0^{\infty}\left[\int_0^{\infty}\max\left\{0,1-\frac{r}{x}\right\}\,dF(x)\right]\,dr\\
&=2\int_0^{\infty}\left[\int_0^{\infty}\max\left\{0,1-\frac{r}{x}\right\}\,dr\right]\,dF(x)
=\int_0^{\infty}x\,dF(x),
\end{align*}
where the interchange of integration order is permissible under the assumption that $E[X]<\infty$.
\end{proof}

\begin{remark}
As a straightforward corollary, $\mathcal{F}[k](0)=E[X]$. 
\end{remark}

The scaling $\rho=A/\tau=E[X]/\tau$ is a key ingredient in parameter tuning when we compare the empirical performance of kernels. Note that with $k(r)$ scaled to $k(\rho r)$, the following facts occur simultaneously for continuous random variables:
\begin{enumerate}
\item The cdf that constructs $k(\rho r)$ is $F(\rho x)$;
\item The corresponding random variable is $X/\rho$;
\item The Fourier transform of $k(\rho r)$ evaluates to $\frac{1}{\rho}\mathcal{F}[k](\frac{t}{\rho})$.
\end{enumerate}
For discrete variables, the same facts hold, too; but be minded that the support is possibly changed (e.g., from integers to real numbers).

\section{Example Kernels}\label{sec:dist}
An application of Polya's characterization is to construct positive-definite functions from known probability distributions. In this section, we consider a number of applicable distributions, either discrete or continuous, and derive explicit formulas for the corresponding kernel $k$ and its Fourier transform. There incur a number of special functions, whose definitions are given in Appendix~\ref{app:special.function}. The definitions generally conform to convention.

\subsection{Constructed from (Shifted) Poisson Distribution}\label{sec:dist.poisson}
If $Y$ is a random variable of the Poisson distribution Pois$(\mu)$ with rate $\mu>0$, we have the following known facts:
\begin{enumerate}
\item pmf $\displaystyle f_Y(x)=\frac{\mu^xe^{-\mu}}{x!}$, $x=0,1,2,\ldots$
\item cdf $\displaystyle F_Y(x)=\frac{\Gamma(\lfloor x+1 \rfloor,\mu)}{\Gamma(\lfloor x+1 \rfloor)}$,
\item mean $E[Y]=\mu$,
\item cf $\varphi_Y(t)=\exp[\mu(e^{\im t}-1)]$,
\end{enumerate}
where $\Gamma(s)$ is the gamma function and $\Gamma(s,t)$ is the upper incomplete gamma function, with $t$ being the lower integration limit (see Appendix~\ref{app:special.function} for the formal definition).

Because the support of $Y$ includes zero, we shift the distribution and define $X=Y+1$, such that the value of the random variable starts from $1$. Then, we have
\[
f_X(x)=f_Y(x-1), \quad F_X(x)=F_Y(x-1), \quad E[X]=E[Y]+1, \quad x=1,2,\ldots
\]

To derive the kernel and its Fourier transform, consider the random variable $\widetilde{X}$ stated in Theorem~\ref{thm:special.case} and subsequently revealed by~\eqref{eqn:C1}. We write
\[
\frac{f_X(x)}{x}=\frac{f_Y(x-1)}{x}=\frac{1}{\mu}\cdot\frac{\mu^xe^{-\mu}}{x!}=\frac{1}{\mu}f_Y(x).
\]
Then, clearly,
\[
C=\frac{1}{\mu} \quad\text{and}\quad \widetilde{X}=Y.
\]
Thus, applying Theorem~\ref{thm:special.case}, we immediately obtain the kernel and the Fourier transform explicitly:
\begin{equation}\label{eqn:k.pois}
k(r)=\begin{dcases}
1-\frac{r}{\mu}\left[1-\frac{\Gamma(\lfloor r+1 \rfloor,\mu)}{\Gamma(\lfloor r+1 \rfloor)}\right], & 0\le r<1,\\
\left[1-\frac{\Gamma(\lfloor r \rfloor,\mu)}{\Gamma(\lfloor r \rfloor)}\right]
-\frac{r}{\mu}\left[1-\frac{\Gamma(\lfloor r+1 \rfloor,\mu)}{\Gamma(\lfloor r+1 \rfloor)}\right], & r\ge1,
\end{dcases}
\end{equation}
and
\begin{equation}\label{eqn:ft.pois}
\mathcal{F}[k](t)=\frac{2-\exp[\mu(e^{\im t}-1)]-\exp[\mu(e^{-\im t}-1)]}{\mu t^2}.
\end{equation}
Note that the constructed kernel $k$ is piecewise linear.

\subsection{Constructed from Gamma Distribution}\label{sec:dist.gamma}
If $X$ is a random variable of the gamma distribution Gamma$(s,\theta)$ with shape $s>0$ and scale $\theta>0$, we have the following known facts:
\begin{enumerate}
\item pdf $\displaystyle f(x)=\frac{x^{s-1}e^{-x/\theta}}{\Gamma(s)\theta^s}$,
\item cdf $\displaystyle F_X(x)=1-\frac{\Gamma(s,x/\theta)}{\Gamma(s)}$,
\item mean $E[X]=\theta s$,
\item cf $\varphi_X(t)=(1-\im\theta t)^{-s}$.
\end{enumerate}

We discuss three cases of the shape $s$. When $s>1$, we write
\[
\frac{f(x)}{x}=\frac{1}{(s-1)\theta}\cdot\frac{x^{s-2}e^{-x/\theta}}{\Gamma(s-1)\theta^{s-1}}.
\]
Then, clearly, with respect to~\eqref{eqn:C2},
\[
C=\frac{1}{(s-1)\theta} \quad\text{and}\quad \widetilde{X}\sim \text{Gamma}(s-1,\theta).
\]
Applying Theorem~\ref{thm:special.case}, we immediately obtain the kernel and the Fourier transform explicitly:
\begin{gather}
k(r)= \frac{\Gamma(s,r/\theta)-r/\theta\cdot\Gamma(s-1,r/\theta)}{\Gamma(s)}, \quad r\ge0, \label{eqn:k.gamma.s}\\
\mathcal{F}[k](t)=\frac{2\left[1-\cos^{s-1}(\omega)\cos((s-1)\omega)\right]}{(s-1)\theta t^2},
\quad\text{with}\quad
\cos(\omega)=(1+\theta^2t^2)^{-\frac{1}{2}}. \label{eqn:ft.gamma.s}
\end{gather}

When $s=1$, the distribution Gamma$(s-1,\theta)$ is undefined. However, we may derive the kernel function directly from~\eqref{eqn:k.alt}:
\begin{equation}\label{eqn:k.gamma.1}
k(r)=e^{-r/\theta}-r/\theta\cdot E_1(r/\theta), \quad r\ge0,
\end{equation}
where $E_1$ is the exponential integral. The Fourier transform admits a closed form due to the known sine transform of $E_1$.

\begin{theorem}
For $k$ defined in~\eqref{eqn:k.gamma.1},
\begin{equation}\label{eqn:ft.gamma.1}
\mathcal{F}[k](t)=
\frac{\log(1+\theta^2t^2)}{\theta t^2}.
\end{equation}
\end{theorem}

\begin{proof}
Based on~\eqref{eqn:k.gamma.1}, we perform a reparameterization $\lambda=1/\theta$ and write
\[
\int_{-\infty}^{\infty}k(r)e^{\im rt}\,dr=
2\int_0^{\infty}e^{-\lambda r}\cos(rt)\,dr
-2\int_0^{\infty}\lambda rE_1(\lambda r)\cos(rt)\,dr.
\]
The first term is a commonly used integral and it is evaluated to
\[
\int_0^{\infty}e^{-\lambda r}\cos(rt)\,dr=\frac{\lambda}{\lambda^2+t^2}.
\]
Then, we perform integration by parts on the second term. Noting that $E'_1(r)=-e^{-r}/r$, we obtain
\begin{align*}
\int_0^{\infty}rE_1(\lambda r)\cos(rt)\,dr
&=\left.rE_1(\lambda r)\frac{\sin(rt)}{t}\right\vert_0^{\infty}
-\int_0^{\infty}-\lambda rE'_1(\lambda r)\frac{\sin(rt)}{t}\,dr
-\int_0^{\infty}E_1(\lambda r)\frac{\sin(rt)}{t}\,dr\\
&=0+\frac{1}{t}\int_0^{\infty}e^{-\lambda r}\sin(rt)\,dr
-\frac{1}{t}\int_0^{\infty}E_1(\lambda r)\sin(rt)\,dr.
\end{align*}
The middle term is a commonly used integral and it is evaluated to
\[
\int_0^{\infty}e^{-\lambda r}\sin(rt)\,dr=\frac{t}{\lambda^2+t^2}.
\]
According to Section 2.11, Equation~(18) of~\cite[p.98]{Bateman1954}, we have for the third term
\[
\int_0^{\infty}E_1(\lambda r)\sin(rt)\,dr=\frac{1}{2t}\log\left(1+\frac{t^2}{\lambda^2}\right).
\]
Combining all these results, we obtain
\[
\int_{-\infty}^{\infty}k(r)e^{\im rt}\,dr=
\frac{\lambda}{t^2}\log\left(1+\frac{t^2}{\lambda^2}\right),
\]
which concludes the theorem.
\end{proof}

When $s<1$, the expression of $k$ in~\eqref{eqn:k.alt} incurs incomplete gamma functions with negative arguments. Such functions are not standard. We therefore do not consider this case. Note, however, that although we do not have an explicit expression for $k$, the results in the preceding section still guarantee that $k$ is a valid kernel.

\subsection{Constructed from Exponential Distribution}\label{sec:dist.exp}
If $X$ is a random variable of the exponential distribution Exp$(\theta)$ with scale $\theta>0$, that is,
\[
f(x)=\frac{1}{\theta}e^{-x/\theta},
\]
then it also belongs to the gamma distribution with shape $s=1$ and scale $\theta$. Hence, the corresponding kernel $k$ and its Fourier transform are given in~\eqref{eqn:k.gamma.1} and~\eqref{eqn:ft.gamma.1} of Section~\ref{sec:dist.gamma}, respectively.

\subsection{Constructed from Chi-Square Distribution}\label{sec:dist.chi2}
If $X$ is a random variable of the chi-square distribution $\chi_{\nu}^2$ with degree of freedom $\nu$, that is,
\[
f(x)=\frac{x^{\nu/2-1}e^{-x/2}}{2^{\nu/2}\Gamma(\nu/2)},
\]
then it also belongs to the gamma distribution with shape $s=\nu/2$ and scale $\theta=2$. In particular, when $\nu=2$, the corresponding kernel $k$ and its Fourier transform are given in~\eqref{eqn:k.gamma.1} and~\eqref{eqn:ft.gamma.1} of Section~\ref{sec:dist.gamma}, respectively. When $\nu>2$, the respective formulas are give in~\eqref{eqn:k.gamma.s} and~\eqref{eqn:ft.gamma.s}.

\subsection{Constructed from Chi Distribution}\label{sec:dist.chi}
If $X$ is a random variable of the chi distribution $\chi_{\nu}$ with degree of freedom $\nu$, we have the following known facts:
\begin{enumerate}
\item pdf $\displaystyle f(x)=\frac{2^{1-\nu/2}}{\Gamma(\nu/2)}x^{\nu-1}e^{-x^2/2}$,
\item cdf $\displaystyle F_X(x)=1-\frac{\Gamma(\nu/2,x^2/2)}{\Gamma(\nu/2)}$,
\item mean $\displaystyle E[X]=\sqrt{2}\frac{\Gamma((\nu+1)/2)}{\Gamma(\nu/2)}$,
\item cf $\displaystyle \varphi_X(t)=M\left(\frac{\nu}{2},\frac{1}{2},\frac{-t^2}{2}\right)+
\im t\sqrt{2}\frac{\Gamma((\nu+1)/2)}{\Gamma(\nu/2)}
M\left(\frac{\nu+1}{2},\frac{3}{2},\frac{-t^2}{2}\right)$,
\end{enumerate}
where $M(a,b,z)$ is Kummer's confluent hypergeometric function.

We discuss two cases of $\nu$. When $\nu>1$, we write
\[
\frac{f(x)}{x}=\frac{\Gamma((\nu-1)/2)}{\sqrt{2}\Gamma(\nu/2)}
\frac{2^{1-(\nu-1)/2}}{\Gamma((\nu-1)/2)}x^{\nu-2}e^{-x^2/2}.
\]
Then, clearly, with respect to~\eqref{eqn:C2},
\[
C=\frac{\Gamma((\nu-1)/2)}{\sqrt{2}\Gamma(\nu/2)}
\quad\text{and}\quad
\widetilde{X}\sim\chi_{\nu-1}.
\]
Applying Theorem~\ref{thm:special.case}, we immediately obtain the kernel and the Fourier transform explicitly:
\begin{gather}
k(r)=\frac{\Gamma(\nu/2,r^2/2)-r/\sqrt{2}\cdot\Gamma((\nu-1)/2,r^2/2)}{\Gamma(\nu/2)}, \quad r\ge0, \label{eqn:k.chi.nu} \\
\mathcal{F}[k](t)=\frac{\sqrt{2}\Gamma((\nu-1)/2)}{t^2\Gamma(\nu/2)}
\left[1-M\left(\frac{\nu-1}{2},\frac{1}{2},\frac{-t^2}{2}\right)\right] \label{eqn:ft.chi.nu}.
\end{gather}

Note that as a special case, when $\nu=2$, the distribution $\chi_{\nu}$ is the same as the Rayleigh distribution with scale $\sigma=1$; see Section~\ref{sec:dist.rayleigh}. The explicit expressions for the kernel and the Fourier transform will be presented therein for a general scale parameter $\sigma$.

When $\nu=1$, the distribution $\chi_{\nu}$ is the same as the half-normal distribution with scale $\sigma=1$; see Section~\ref{sec:dist.half.normal}. The explicit expressions will be presented therein for a general $\sigma$.

\subsection{Constructed from Half-Normal Distribution}\label{sec:dist.half.normal}
If $X$ is a random variable of the half-normal distribution HN$(\sigma)$ with scale $\sigma>0$, we have the following known facts:
\begin{enumerate}
\item pdf $\displaystyle f(x)=\frac{\sqrt{2}}{\sigma\sqrt{\pi}}\exp\left(-\frac{x^2}{2\sigma^2}\right)$,
\item cdf $\displaystyle F_X(x)=\erf\left(\frac{x}{\sigma\sqrt{2}}\right)$,
\item mean $\displaystyle E[X]=\frac{\sigma\sqrt{2}}{\sqrt{\pi}}$,
\item cf $\varphi_X(t)=e^{-\sigma^2t^2/2}[1-\im\erfi(\sigma t/\sqrt{2})]$,
\end{enumerate}
where $\erf$ is the error function and $\erfi$ is the imaginary error function (which, in fact, is a real-valued function when the argument is real).

We may not apply Theorem~\ref{thm:special.case} to derive the explicit formula for $k$, because $C$ is infinite. However, with a change of variable $y=x^2/(2\sigma^2)$, we see that a part of~\eqref{eqn:k.alt} is evaluated to
\[
\int_r^{\infty}\frac{f(x)}{x}\,dx
=\int_r^{\infty}\frac{\sqrt{2}}{x\sigma\sqrt{\pi}}\exp\left(-\frac{x^2}{2\sigma^2}\right)\,dx
=\frac{1}{\sigma\sqrt{2\pi}}\int_{r^2/(2\sigma^2)}^{\infty}\frac{e^{-y}}{y}\,dy
=\frac{E_1(r^2/(2\sigma^2))}{\sigma\sqrt{2\pi}}.
\]
Therefore, an explicit expression for the kernel is
\begin{equation}\label{eqn:k.half.normal}
k(r)=\erfc\left(\frac{r}{\sigma\sqrt{2}}\right)-\frac{1}{\sqrt{\pi}}
\left(\frac{r}{\sigma\sqrt{2}}\right)
E_1\left(\frac{r^2}{2\sigma^2}\right), \quad r\ge0,
\end{equation}
where $\erfc$ is the complementary error function. The following theorem gives the Fourier transform of $k$ in the form of a sine transform, which unfortunately is hard to be further simplified.

\begin{theorem}
For $k$ defined in~\eqref{eqn:k.half.normal},
\begin{equation}\label{eqn:ft.half.normal}
\mathcal{F}[k](t)
=\frac{2}{t\sqrt{\pi}}\int_0^{\infty}E_1(r^2)\sin(\sigma\sqrt{2}tr)\,dr.
\end{equation}
\end{theorem}

\begin{proof}
We first turn $k(r)$ to $k(\sigma\sqrt{2}r)$ in order to simply the math:
\begin{align}
\int_{-\infty}^{\infty}k(r)e^{\im rt}\,dr
&=2\int_0^{\infty}k(r)\cos(rt)\,dr
=2\sqrt{2}\sigma\int_0^{\infty}k(\sigma\sqrt{2}r)\cos(\sigma\sqrt{2}rt)\,dr \nonumber\\
&=2\sqrt{2}\sigma\left[\int_0^{\infty}\erfc(r)\cos(rT)\,dr
-\frac{1}{\sqrt{\pi}}\int_0^{\infty}rE_1(r^2)\cos(rT)\,dr\right], \label{eqn:t1}
\end{align}
where $T=\sigma\sqrt{2}t$.
For the first term, we rearrange the order of integration:
\begin{align}
\int_0^{\infty}\erfc(r)\cos(rT)\,dr
&=\int_0^{\infty}\frac{2}{\sqrt{\pi}}\int_r^{\infty}e^{-x^2}dx\cos(rT)\,dr \nonumber\\
&=\frac{2}{\sqrt{\pi}}\int_0^{\infty}e^{-x^2}\int_0^x\cos(rT)\,drdx
=\frac{2}{T\sqrt{\pi}}\int_0^{\infty}e^{-x^2}\sin(xT)\,dx. \label{eqn:t2}
\end{align}
For the second term, we perform integration by parts:
\begin{align}
\int_0^{\infty}rE_1(r^2)\cos(rT)\,dr
&=\left.\frac{rE_1(r^2)\sin(rT)}{T}\right\vert_0^{\infty}
-\int_0^{\infty}\frac{\sin(rT)}{T}\left[E_1(r^2)-2e^{-r^2}\right]\,dr \nonumber\\
&=-\frac{1}{T}\int_0^{\infty}\sin(rT)E_1(r^2)\,dr
+\frac{2}{T}\int_0^{\infty}\sin(rT)e^{-r^2}\,dr. \label{eqn:t3}
\end{align}
Substituting~\eqref{eqn:t2} and~\eqref{eqn:t3} into~\eqref{eqn:t1}, we obtain the result of the theorem.
\end{proof}

\subsection{Constructed from Rayleigh Distribution}\label{sec:dist.rayleigh}
If $X$ is a random variable of the Rayleigh distribution Rayleigh$(\sigma)$ with scale $\sigma>0$, we have the following known facts:
\begin{enumerate}
\item pdf $\displaystyle f(x)=\frac{x}{\sigma^2}e^{-x^2/(2\sigma^2)}$,
\item cdf $F_X(x)=1-e^{-x^2/(2\sigma^2)}$,
\item mean $\displaystyle E[X]=\sigma\sqrt{\frac{\pi}{2}}$.
\end{enumerate}

To derive explicit expressions, we note that
\[
\frac{f(x)}{x}=\frac{\sqrt{\pi}}{\sigma\sqrt{2}}\frac{\sqrt{2}}{\sigma\sqrt{\pi}}e^{-x^2/(2\sigma^2)}.
\]
Then, clearly, with respect to~\eqref{eqn:C2},
\[
C=\frac{1}{\sigma}\sqrt{\frac{\pi}{2}}
\quad\text{and}\quad
\widetilde{X}\sim \text{HN}(\sigma).
\]
Applying Theorem~\ref{thm:special.case} with the known facts for the half-normal distribution listed in Section~\ref{sec:dist.half.normal}, we immediately obtain the kernel and the Fourier transform explicitly:
\begin{gather}
k(r)=\exp\left(-\frac{r^2}{2\sigma^2}\right)
-\sqrt{\pi}\left(\frac{r}{\sigma\sqrt{2}}\right)
\erfc\left(\frac{r}{\sigma\sqrt{2}}\right), \quad r\ge0, \label{eqn:k.rayleigh}\\
\mathcal{F}[k](t)=\frac{\sqrt{2\pi}}{\sigma t^2}\left[1-\exp\left(-\frac{\sigma^2t^2}{2}\right)\right]. \label{eqn:ft.rayleigh}
\end{gather}

\subsection{Constructed from Nakagami Distribution}\label{sec:dist.nakagami}
If $X$ is a random variable of the Nakagami distribution Nakagami$(m,\Omega)$ with shape $m\ge1/2$ and spread $\Omega>0$, that is,
\[
f(x)=\frac{2m^mx^{2m-1}e^{-mx^2/\Omega}}{\Gamma(m)\Omega^m},
\]
we may perform a reparameterization
\[
m=\nu/2,\quad \Omega = \nu\theta^2,
\]
and obtain
\[
f(x)=\frac{1}{\theta}\cdot\frac{2^{1-\nu/2}}{\Gamma(\nu/2)}(x/\theta)^{\nu-1}e^{-(x/\theta)^2/2}.
\]
Clearly, $f$ is a rescaling of the pdf of the chi distribution $\chi_{\nu}$, with the integer $\nu$ (degree of freedom) relaxed to a real number.

We discuss two cases of $m$. When $m>1/2$ (i.e., $\nu>1$), we will reuse the formulas~\eqref{eqn:k.chi.nu} and~\eqref{eqn:ft.chi.nu} derived for $\chi_{\nu}$. The reason why~\eqref{eqn:k.chi.nu} and~\eqref{eqn:ft.chi.nu} are valid for non-integers $\nu$ is that they are derived from the cdf and the cf of $\chi_{\nu}$, wherein the integration results are valid for any real numbers $\nu>1$. Then, with a proper scaling, we have for the Nakagami distribution:
\begin{gather}
k(r)=\frac{\Gamma(m,mr^2/\Omega)-\sqrt{m}r/\sqrt{\Omega}\cdot\Gamma(m-1/2,mr^2/\Omega)}{\Gamma(m)}, \quad r\ge0, \label{eqn:k.nakagami.m}\\
\mathcal{F}[k](t)=2\sqrt{\frac{m}{\Omega}}\frac{\Gamma(m-1/2)}{t^2\Gamma(m)}
\left[1-M\left(m-\frac{1}{2},\frac{1}{2},\frac{-\Omega t^2}{4m}\right)\right]. \label{eqn:ft.nakagami.m}
\end{gather}

When $m=1/2$, the distribution is the same as the half-normal distribution with scale $\sigma=\sqrt{\Omega}$. Then, substituting  $\sigma=\sqrt{\Omega}$ into~\eqref{eqn:k.half.normal} and~\eqref{eqn:ft.half.normal}, we have
\begin{gather}
k(r)=\erfc\left(\frac{r}{\sqrt{2\Omega}}\right)-\frac{1}{\sqrt{\pi}}
\left(\frac{r}{\sqrt{2\Omega}}\right)
E_1\left(\frac{r^2}{2\Omega}\right), \quad r\ge0, \label{eqn:k.nakagami.1}\\
\mathcal{F}[k](t)=\frac{2}{t\sqrt{\pi}}\int_0^{\infty}E_1(r^2)\sin(\sqrt{2\Omega}tr)\,dr. \label{eqn:ft.nakagami.1}
\end{gather}

\subsection{Constructed from Weibull Distribution}\label{sec:dist.weibull}
If $X$ is a random variable of the Weibull distribution Weibull$(\theta,\alpha)$ with scale $\theta>0$ and shape $\alpha>0$, we have the following known facts:
\begin{enumerate}
\item pdf $\displaystyle f(x)=\frac{\alpha}{\theta}\left(\frac{x}{\theta}\right)^{\alpha-1}e^{-(x/\theta)^{\alpha}}$,
\item cdf $F_X(x)=1-e^{-(x/\theta)^{\alpha}}$,
\item mean $E[X]=\theta\Gamma(1+1/\alpha)$.
\end{enumerate}

We discuss two cases of $\alpha$. When $\alpha>1$, with a change of variable $y=(x/\theta)^{\alpha}$, we see that a part of~\eqref{eqn:k.alt} is evaluated to
\[
\int_r^{\infty}\frac{f(x)}{x}\,dx=
\int_r^{\infty}\frac{\alpha}{\theta}\left(\frac{x}{\theta}\right)^{\alpha-1}e^{-(x/\theta)^{\alpha}}\frac{1}{x}\,dx=
\frac{1}{\theta}\int_{(r/\theta)^{\alpha}}^{\infty}y^{-1/\alpha}e^{-y}\,dy=
\frac{1}{\theta}\Gamma(1-1/\alpha,(r/\theta)^{\alpha}).
\]
Therefore, an explicit expression for the kernel is
\begin{equation}\label{eqn:k.weibull.alpha}
k(r)=e^{-(r/\theta)^{\alpha}}-(r/\theta)\Gamma(1-1/\alpha,(r/\theta)^{\alpha}).
\end{equation}
We do not have an explicit expression for the Fourier transform, unfortunately.

When $\alpha=1$, the distribution is the same as the exponential distribution with scale $\theta$; it is also the same as the gamma distribution with shape $s=1$ and scale $\theta$. Hence, the corresponding kernel $k$ and its Fourier transform are given in~\eqref{eqn:k.gamma.1} and~\eqref{eqn:ft.gamma.1} of Section~\ref{sec:dist.gamma}, respectively.

\subsection{Summary}
We summarize the results obtained so far in Table~\ref{tab:dist.kernel} (located after the bibliography). This table lists many applicable distributions and the correspondingly constructed kernels. Accompanied with the distributions are the pmf/pdf's and the mean's. The pmf/pdf's are used to uniquely identify the distributions, because different authors may call the parameters differently. Moreover, for the Poisson distribution, it has been shifted to avoid a nonzero mass at the origin. Hence, one is suggested to fully digest the notations before usage. The mean's are used to standardize a kernel so that the area under curve is $1$ (see Section~\ref{sec:scaling}). Accompanied with the kernels are the explicit expressions for $k$ and the Fourier transform $\mathcal{F}[k]$. These expressions could be used, for example, for further deriving analytic properties.

The table consists of three parts. The top part contains a discrete distribution, whereas the other two parts contain continuous ones. The distributions in the bottom part are special cases of those in the middle. The equivalence is indicated in the last column. Therefore, we consider that practical use of the distributions focuses mainly on the top and middle parts of the table.

A practical aspect for the use of the distributions is the choice of parameters, which is reflected in the last column. All continuous distributions therein contain a ``scale'' parameter that affects the spread in one way or another. Because we use the distribution mean $A=E[X]$ to perform standardization, we may fix the scale parameter at an arbitrary value (particularly, $1$) and let the actual spread be determined by a scaling factor $\rho=A/\tau$ where $\tau$ is tuned (see Section~\ref{sec:scaling}). Apart from the scale parameter, some distributions come additionally with a ``shape'' parameter, which appears as an exponent for $x$ in the pdf. When tuning such a parameter, one may search for an optimal one from a grid (e.g., integers and half-integers). The same practice applies to the ``rate'' parameter of Poisson.

Figures~\ref{fig:kernel.var} and~\ref{fig:kernel.var2} (located after the bibliography) plot the kernels listed on the top and the middle parts of Table~\ref{tab:dist.kernel}, with several choices of a parameter ($\mu$ in Poisson, $s$ in gamma, $m$ in Nakagami, and $\alpha$ in Weibull). As expected, the kernels are all convex and monotonically decreasing from $1$ to $0$. The right column of the figure shows the kernels scaled by the distribution mean; therefore, the area under curve is $1$. These curves smoothly vary with the parameter.

\section{Random Feature Maps}\label{sec:var}
Mercer's theorem~\cite{Mercer1909} guarantees that there exists a feature map $z(x)$ such that $k(x-x')$ is equal to the inner product $\langle z(x),z(x')\rangle$, where $z$ is a finite-dimensional or countably infinite-dimensional vector. The random feature approaches construct such maps so that $z(x)$ is random and that the expectation of $\langle z(x),z(x')\rangle$ is equal to $k(x-x')$. Naturally, one may define $D$ independent copies of $z$, namely, $z^{(l)}$ for $l=1,\ldots,D$, and use the Monte Carlo sample average
$\frac{1}{D}\sum_{l=1}^D\langle z^{(l)}(x),z^{(l)}(x')\rangle$
to reduce the randomness of the inner product as an unbiased approximation to the kernel $k$. In this section, we compare the randomness of different approaches.

On notation: The data $x$ in the general case is a vector; however in some cases (e.g., random binning), the kernel acts on a scalar input $x$. The feature map $z$ may be a scalar-valued or a vector-valued map, depending on context. We define a random function $\tilde{k}$ as a shorthand notation of the inner product:
\[
\tilde{k}(x,x')\equiv\langle z(x),z(x')\rangle
\]
and write $\widetilde{K}$ as the corresponding kernel matrix. Note that although $k$ is stationary, $\tilde{k}$ may not (hence the notations are $k(x-x')$ and $\tilde{k}(x,x')$, respectively). Then, with $D$ independent copies of the feature maps, the corresponding kernel matrix becomes $\frac{1}{D}\sum_{l=1}^D\widetilde{K}^{(l)}$. We are interested in the probabilistic properties of $\frac{1}{D}\sum_{l=1}^D\widetilde{K}^{(l)}-K$. Because the dimension of the data matters only in the Fourier transform of the kernel, the mathematical derivation here focuses on one-dimensional kernel functions. Generalizations to the multidimensional case are straightforward. The theorems in this section are presented to be applicable to the multidimensional case, too.

\subsection{Random Fourier Map}\label{sec:rff}
The random Fourier approach defines the feature map $z(x)=e^{\im wx}$, where $w$ is drawn from the cdf $F$ in~\eqref{eqn:rff}. Then, the same equation immediately verifies that the inner product $\langle z(x),z(x')\rangle=e^{\im w(x-x')}$ has an expectation $k(x-x')$. Additionally, we easily obtain that the variance of the inner product is
\begin{equation}\label{eqn:fourier.var}
\var[\langle z(x),z(x')\rangle]
=\left[\int |e^{\im w(x-x')}|^2\,dF(w)\right]
-k(x-x')^2
=1-k(x-x')^2.
\end{equation}

In practice, it is often more desirable to use a feature map that is real-valued. Hence, the real version of the map is $z(x)=\sqrt{2}\cos(wx+b)$, where $b$ is drawn from $\mathcal{U}(0,2\pi)$. This map still yields expectation $k(x-x')$ for the inner product:
\begin{align*}
E[\langle z(x),z(x')\rangle]
&=\int_{-\infty}^{\infty}\int_0^{2\pi}
\Big(2\cos(wx+b)\cos(wx'+b)\Big)\frac{1}{2\pi}\,db\,dF(w)\\
&=\int_{-\infty}^{\infty}\cos(w(x-x'))\,dF(w)
=k(x-x'),
\end{align*}
but gives a larger variance:
\begin{align}
\var[\langle z(x),z(x')\rangle]
&=\left[\int_{-\infty}^{\infty}\int_0^{2\pi}
\Big(2\cos(wx+b)\cos(wx'+b)\Big)^2\frac{1}{2\pi}\,db\,dF(w)\right]
-k(x-x')^2 \nonumber \\
&=\left[\int_{-\infty}^{\infty}
\Big(1+\frac{1}{2}\cos(2w(x-x'))\Big)\,dF(w)\right]
-k(x-x')^2 \nonumber\\
&=1+\frac{1}{2}k(2(x-x'))-k(x-x')^2.\label{eqn:fourier.var2}
\end{align}

The feature map is straightforwardly generalized to the multidimensional case through multidimensional Fourier transform, the details of which are omitted here. With one further generalization---using a Monte Carlo sample average of $D$ independent copies to replace $z$---we arrive at the following result. It states that the random Fourier approach gives an unbiased approximation. It also gives the squared Frobenius norm error of the approximation.

\begin{theorem}\label{thm:rff}
Let $K$ be the kernel matrix of a kernel $k$ on data points $x_i$, $i=1,\ldots,n$. Let $\widetilde{K}^{(l)}$, $l=1,\ldots,D$ be the kernel matrices resulting from $D$ independent random Fourier feature maps for $k$. We have
\[
E\left[\frac{1}{D}\sum_{l=1}^D\widetilde{K}^{(l)}\right]=K.
\]
Moreover, for the complex feature map,
\[
E\left[\Bigg\|\frac{1}{D}\sum_{l=1}^D\widetilde{K}^{(l)}-K\Bigg\|_F^2\right]=\frac{1}{D}(n^2-\|K\|_F^2),
\]
and for the real feature map,
\[
E\left[\Bigg\|\frac{1}{D}\sum_{l=1}^D\widetilde{K}^{(l)}-K\Bigg\|_F^2\right]=\frac{1}{D}\left(n^2+\frac{1}{2}\sum_{i,j=1}^nk(2(x_i-x_j))-\|K\|_F^2\right).
\]
\end{theorem}

\begin{proof}
The first expectation is obvious and the second one is analogous to the third one. Thus, we prove only the second one. By the linearity of expectation, we have
\[
E\left[\Bigg\|\frac{1}{D}\sum_{l=1}^D\widetilde{K}^{(l)}-K\Bigg\|_F^2\right]
=\sum_{i,j=1}^nE\left[\Bigg(\frac{1}{D}\sum_{l=1}^D\widetilde{K}^{(l)}_{ij}-K_{ij}\Bigg)^2\right].
\]
Note that inside the summation, each expectation is nothing but the variance of
$\frac{1}{D}\sum_{l=1}^D\widetilde{K}^{(l)}_{ij}$. Then, with independence,
\[
E\left[\Bigg(\frac{1}{D}\sum_{l=1}^D\widetilde{K}^{(l)}_{ij}-K_{ij}\Bigg)^2\right]
=\var\left[\frac{1}{D}\sum_{l=1}^D\widetilde{K}^{(l)}_{ij}\right]
=\frac{1}{D}\var[\widetilde{K}^{(1)}_{ij}].
\]
By~\eqref{eqn:fourier.var}, we see that
$
\var[\widetilde{K}^{(1)}_{ij}]=1-k(x_i-x_j)^2,
$
which proves the second expectation in the theorem.

For the third expectation, follow the same argument and apply~\eqref{eqn:fourier.var2} at the end.
\end{proof}

\subsection{Random Binning Map}\label{sec:rbf}
The random binning approach applies to multidimensional kernel functions $k$ that are a tensor product of one-dimensional kernels. The approach was originally proposed for only the exponential kernel, because based on~\eqref{eqn:rbf}, the term $wk''(w)$ happens to be a known pdf (gamma distribution of a certain shape). One easily generalizes the approach based on, instead, \eqref{eqn:k}, through a reverse thinking: any cdf corresponds to a valid kernel. Hence, in the general setting, we consider the following construction, which defines a marginal distribution for the inner product $\tilde{k}=\langle z(x),z(x')\rangle$:
\begin{enumerate}
\item Let $F(w)$ be a cdf with positive support.
\item Let a random one-dimensional grid have spacing $w$ and offset $b$, where $w\sim F(w)$ and $b\sim\mathcal{U}(0,w)$. In other words, we have the conditional probability density $f(b|w)=w^{-1}$.
\item Define the feature vector $z(x)$, one element for each grid bin, that takes $1$ when $x$ falls in the bin and $0$ otherwise. For two points $x$ and $x'$, because the probability that they fall in the same bin is $\max\{0,1-r/w\}$ with $r=|x-x'|$, we have the conditional probability
\[
\Pr(\tilde{k}=1\mid w,b)=\max\left\{0,1-\frac{r}{w}\right\},\qquad
\Pr(\tilde{k}=0\mid w,b)=1-\Pr(\tilde{k}=1\mid w,b).
\]
\end{enumerate}
Therefore, this procedure defines a marginal distribution for $\tilde{k}$ whose pmf is
\[
\Pr(\tilde{k}=1)=\int_0^{\infty}\int_0^w\Pr(\tilde{k}=1\mid w,b)f(b|w)\,db\,dF(w)
=k(r),
\qquad\text{(cf.~\eqref{eqn:k})}
\]
and $\Pr(\tilde{k}=1)=1-\Pr(\tilde{k}=0)$. In other words, $\tilde{k}$ is a Bernoulli variable with success probability $k$; hence, obviously,
\begin{equation}\label{eqn:binning.var}
E[\tilde{k}]=k \quad\text{and}\quad \var[\tilde{k}]=k-k^2.
\end{equation}

This feature map is straightforwardly generalized to the multidimensional case through using a multidimensional grid. With one further generalization---using a Monte Carlo sample average of $D$ independent copies to replace $z$---we arrive at the following result, parallel to Theorem~\ref{thm:rff}.

\begin{theorem}\label{thm:rbf}
Let $K$ be the kernel matrix of a kernel $k$ and let $\widetilde{K}^{(l)}$, $l=1,\ldots,D$ be the kernel matrices resulting from $D$ independent random binning feature maps for $k$. We have
\[
E\left[\frac{1}{D}\sum_{l=1}^D\widetilde{K}^{(l)}\right]=K,
\]
and
\[
E\left[\Bigg\|\frac{1}{D}\sum_{l=1}^D\widetilde{K}^{(l)}-K\Bigg\|_F^2\right]=\frac{1}{D}\left(\sum_{i,j=1}^nK_{ij}-\|K\|_F^2\right).
\]
\end{theorem}

\begin{proof}
The proof is analogous to that of Theorem~\ref{thm:rff}, except that at the end we apply
$
\var[\widetilde{K}^{(1)}_{ij}]=k(x_i-x_j)-k(x_i-x_j)^2.
$
\end{proof}

\subsection{Discussions}\label{sec:discuss}
Theorems~\ref{thm:rff} and~\ref{thm:rbf} indicate that for a kernel $k$ that admits both random Fourier and random binning feature maps, the latter map results in an approximate kernel matrix closer to $K$ than does the former map, if the same sample size $D$ is used, because $0\le K_{ij}\le1$. Moreover, \eqref{eqn:fourier.var}, \eqref{eqn:fourier.var2}, and~\eqref{eqn:binning.var} reveal that such a better approximation is elementwise. Of course, a better quality in matrix approximation does not necessarily imply a superior performance in a machine learning task, where the performance metric might not be directly connected with matrix approximation. In practice, however, our experience shows that random binning indeed performs better almost always, in the sense that it requires a (much) smaller $D$ for a matching regression error/classification accuracy, compared with random Fourier. See experimental results in the next section.

One advantage of random Fourier, though, is that it generalizes more broadly to multidimensional inputs, through multidimensional Fourier transforms. As long as the respective multivariate probability distribution can be easily sampled from, the random Fourier map is efficient to compute. Such is the case, for example, for the squared exponential kernel (also called the Gaussian kernel), because the corresponding distribution is multivariate normal. As another example, the exponential kernel (note the vector norm)
\[
\exp\left(-\frac{\|x-x'\|_2}{\sigma}\right)
\]
is corresponded by multivariate Cauchy. In fact, both the squared exponential and the exponential kernels are special cases of the Mat\'{e}rn family of kernels~\cite{Stein1999,Rasmussen2006}, whose corresponding distributions are the multivariate t-distributions, when the Mat\'{e}rn smoothness parameter is an integer or a half-integer.

On the other hand, random binning is applicable to only tensor-product kernels; e.g., the Laplace kernel
\[
\exp\left(-\frac{\|x-x'\|_1}{\sigma}\right)
=\exp\left(-\frac{|(x)_1-(x')_1|}{\sigma}\right)
\exp\left(-\frac{|(x)_2-(x')_2|}{\sigma}\right)\cdots
\exp\left(-\frac{|(x)_d-(x')_d|}{\sigma}\right),
\]
where $(\cdot)_i$ is used to index the dimensions, not the data. Such is not the limitation of Polya's criterion, because one may easily generalize~\eqref{eqn:k} to the multidimensional case by using a multidimensional positive-definite function to replace the triangular function in the integrand. However, the challenge is that if the integrand is not a tensor product, it is difficult to define a ``bin'' such that two points fall in the same bin with a probability equal to the integrand.

In the next section, we will perform an experiment that compares also the Gaussian kernel as an example kernel for random Fourier, which, despite the aforementioned advantage, performs less well than random binning.

\section{Numerical Experiments}\label{sec:exp}
In this section, we demonstrate the empirical performance of random Fourier (denoted by ``RF'') and random binning (denoted by ``RB''), as kernel approximation approaches for regression and classification, in the reproducing kernel Hilbert space (RKHS). We perform the experiments with eight benchmark data sets downloaded from \url{http://www.csie.ntu.edu.tw/~cjlin/libsvm/}. The primary reason of using these data sets is their varying sizes $n$ and dimensions $d$. Some of the data sets come with a train/test split; for those not, we performed a 4:1 split. Attributes were normalized to $[-1,1]$. Table~\ref{tab:dataset} gives the detailed information.

\begin{table}[ht]
\centering
\caption{Data sets.}
\label{tab:dataset}
\begin{tabular}{ccrrrcccrrr}
\\[-5pt]
\hline
Name & Type & $d$ & $n$ Train & $n'$ Test\\
\hline
cadata            & regression            &   8 &    16,512 &     4,128\\
YearPredictionMSD & regression            &  90 &   463,518 &    51,630\\
\hline
ijcnn1            & binary classification &  22 &    35,000 &    91,701\\
covtype.binary    & binary classification &  54 &   464,809 &   116,203\\
SUSY              & binary classification &  18 & 4,000,000 & 1,000,000\\
\hline
mnist             & 10 classes            & 780 &    60,000 &    10,000\\
acoustic          & 3 classes             &  50 &    78,823 &    19,705\\
covtype           & 7 classes             &  54 &   464,809 &   116,203\\
\hline
\end{tabular}
\end{table}

\subsection{Matrix Approximation}\label{sec:mat.approx}
The purpose of the following experiment is to empirically verify Theorems~\ref{thm:rff} and~\ref{thm:rbf} regarding the kernel matrix approximation error, and show how large the gap could be between the two feature maps. For this, we use the Laplace kernel (tensor product of one-dimensional exponential kernels) as an example, because the two corresponding distributions, Cauchy for the random Fourier map and gamma for the random binning map, can be easily sampled from.

The examples are run on the three small data sets listed in Table~\ref{tab:dataset}---cadata, ijcnn1, and acoustic---whose full kernel matrices (sizes on the order $\sim10^4$ to $10^5$) are affordable to compute. Before running a machine learning task, we do not know the optimal scale parameter $\sigma$ in the kernel $k(r)=e^{-r/\sigma}$. Hence, we fix $\sigma=1$ as a reasonable choice. For a related experiment that uses the tuned $\sigma$, see Section~\ref{sec:approx.vs.predict}.

In Figure~\ref{fig:approx.err}, we plot the relative Frobenius norm error, defined as
\begin{equation}\label{eqn:F-norm.err}
E\left[\left.\Bigg\|\frac{1}{D}\sum_{l=1}^D\widetilde{K}^{(l)}-K\Bigg\|_F^2
\right/\|K\|_F^2\right]^{1/2},
\end{equation}
in straight lines. This quantity is similar to the so-called ``standard deviation to mean ratio'' in standard statistics. The lines are computed according to the results given by Theorems~\ref{thm:rff} and~\ref{thm:rbf}. Then, we plot the actual error (with the the expectation sign in~\eqref{eqn:F-norm.err} removed) as scattering crosses, overlaid with the lines.

\begin{figure}[ht]
\centering
\subfigure[cadata]{
  \includegraphics[width=0.32\linewidth]{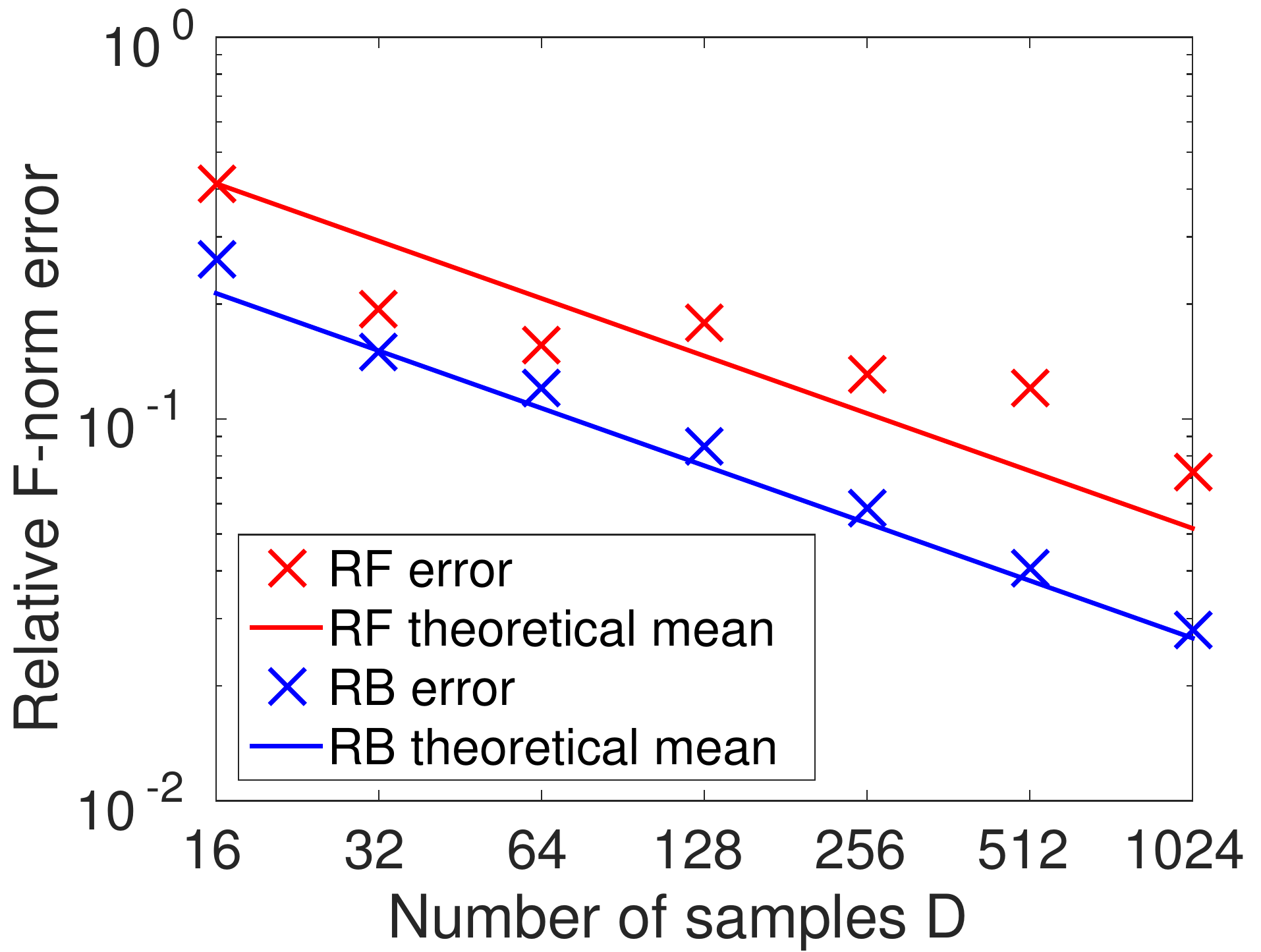}}
\subfigure[ijcnn1]{
  \includegraphics[width=0.32\linewidth]{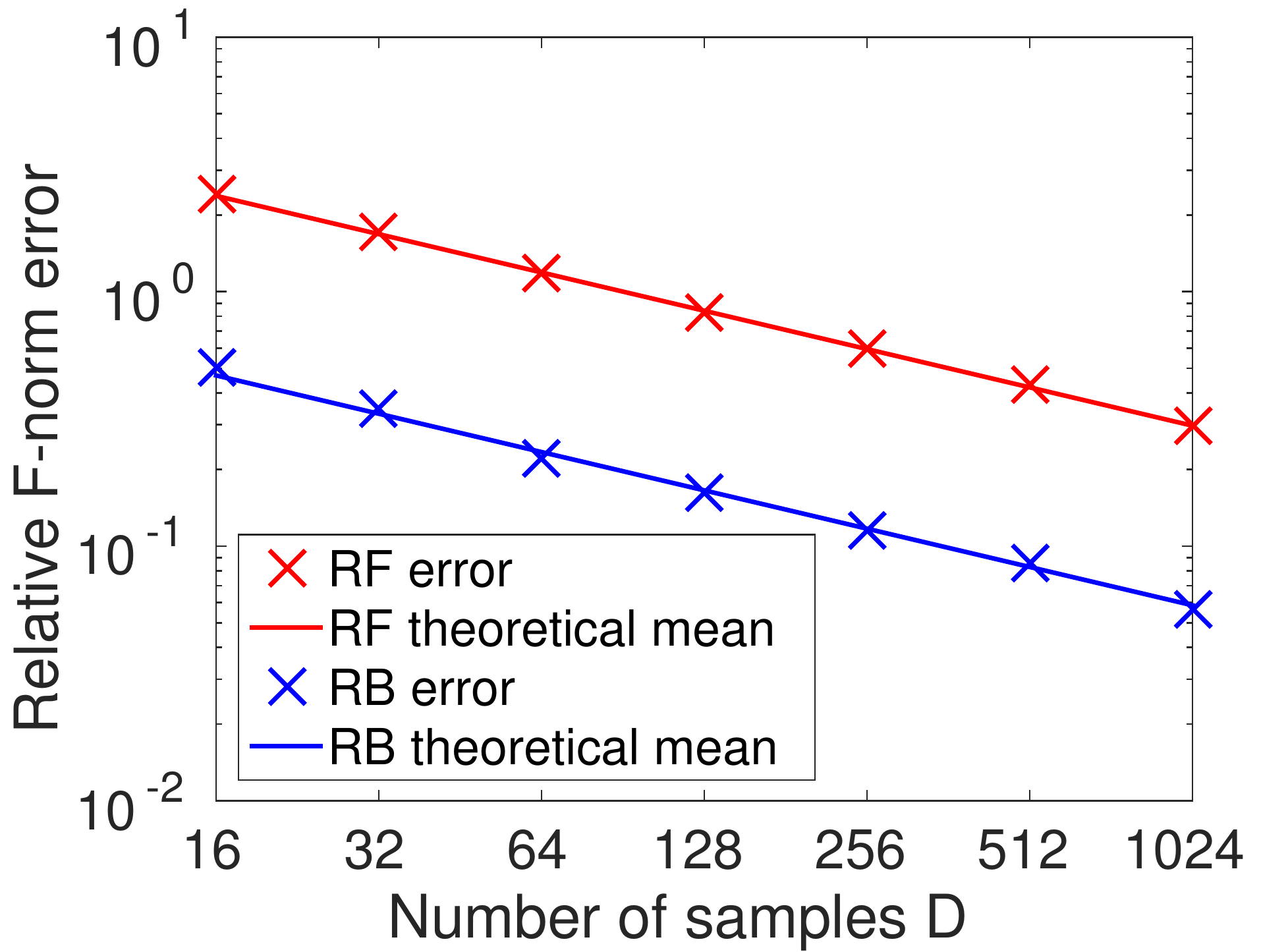}}
\subfigure[acoustic]{
  \includegraphics[width=0.32\linewidth]{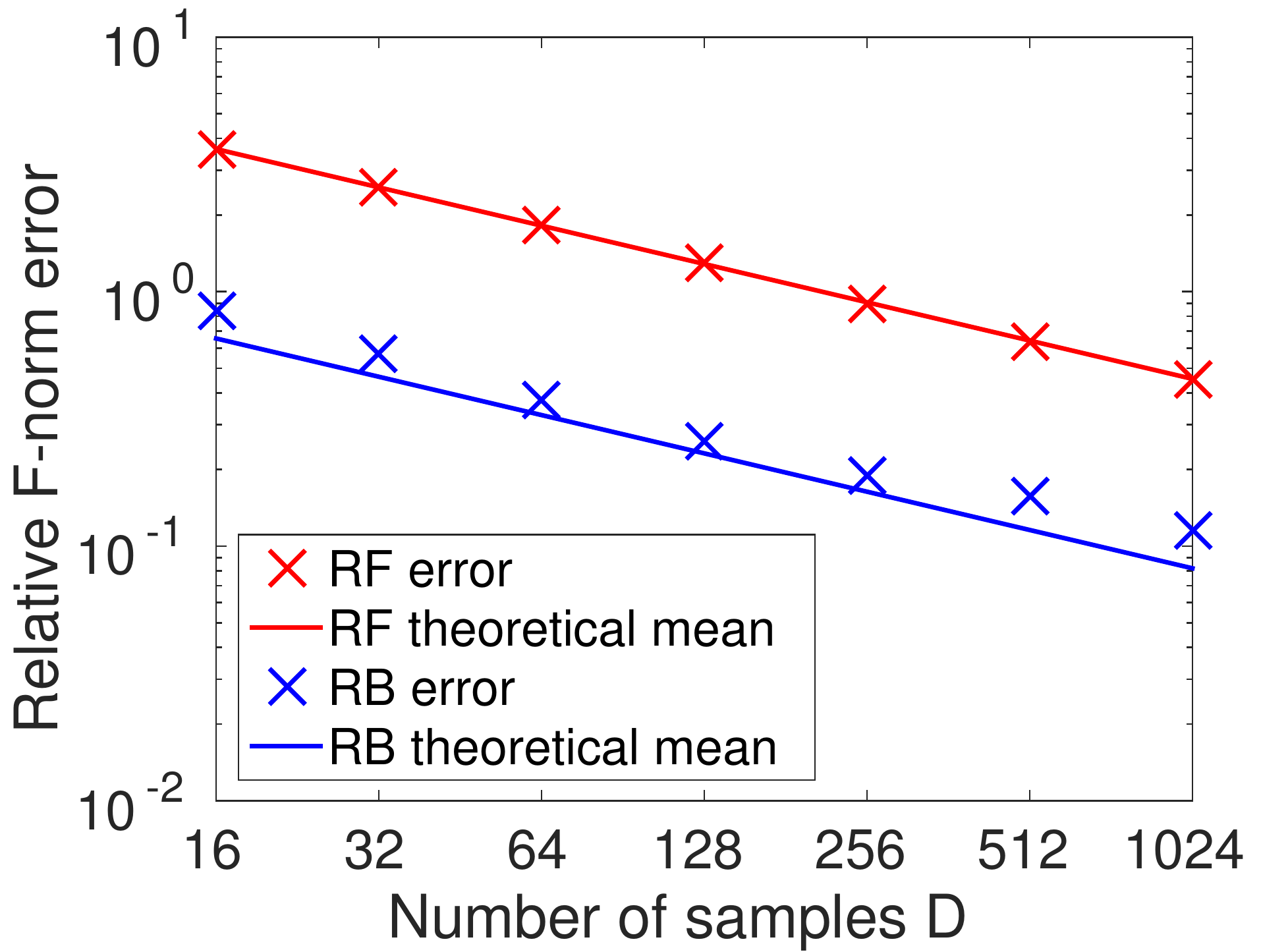}}
\caption{Matrix approximation error as a function of the sample size $D$.}
\label{fig:approx.err}
\end{figure}

One sees that the actual error is well aligned with the theoretical mean. Furthermore, there is a clear gap between the two feature maps; random binning always yields a smaller error. The largest gap corresponds to almost a one-digit difference. Clearly, for different data sets, the gap may be different; and even for the same data set, the gap may also vary when the scale parameter $\sigma$ varies. The spirit of this experiment, after all, is that the theoretical analysis gives a clear preference to random binning, empirically verified.

\subsection{Regression/Classification}\label{sec:exp.perf}
In the next experiment, we apply the random feature maps for regression and classification in the RKHS. The unified setting is that given data $\{x_i\}_{i=1}^n$ with targets $\{y_i\}_{i=1}^n$, we minimize the risk functional
\[
\mathcal{L}(f)=\sum_{i=1}^nV(f(x_i),y_i)+\lambda\langle f,f\rangle_{\mathcal{H}_k}
\]
within the RKHS $\mathcal{H}_k$ defined by a kernel function $k$, where $V(\cdot,\cdot)$ is a loss function, $\langle\cdot,\cdot\rangle$ is the inner product associated to $\mathcal{H}_k$, and $\lambda$ is a regularization parameter. We choose to use the squared loss $V(t,y)=(t-y)^2$ as in~\cite{Rahimi2007}, because due to the Representer Theorem~\cite{Kimeldorf1970,Schoelkopf2001}, the optimal function admits a well-known closed-form expression
\[
f(x)=k_x(K+\lambda I)^{-1}y,
\]
where $k_x$ is the row vector of $k(x,x_i)$ for all $i$ and $y$ is the column vector of all targets. We will use the approximate kernel $\tilde{k}$, defined as the Euclidean inner product of the random feature maps, to replace a kernel $k$.

We perform the experiment with all data sets listed in Table~\ref{tab:dataset}. The performance metric is mean squared error (MSE) for regression and accuracy for classification. Parameters are tuned through cross validation. In particular, for random binning, the scaling factor $\rho=E[X]/\tau$ for the kernel is obtained through actually tuning the assisting parameter $\tau$, as discussed in depth in Section~\ref{sec:scaling}.

We compare random Fourier with random binning, by using two kernels for the former map (Laplace and Gaussian) and four kernels for the latter (those constructed from shifted Poisson, gamma, Nakagami, and Weilbull distributions). Note that the Laplace kernel is equivalent to the one constructed from gamma distribution according to~\eqref{eqn:k}, with a particular shape $s=2$. However, for the random binning map, the shape is considered a tuning parameter, which is not the case for the random Fourier map.

Figure~\ref{fig:regres.classi} plots the regression/classification performance when the sample size $D$ increases. One sees that the performance curves for the two random feature maps are separately clustered in general. The curves of random binning clearly indicate a better performance than do those of random Fourier. Table~\ref{tab:param} lists the tuned parameters that generate the results of Figure~\ref{fig:regres.classi}. We display only those for random binning, because the parameters for random Fourier vary significantly when the number $D$ of samples changes. One observation from the table is that the optimal shape $s$ of the gamma distribution is not always achieved by $2$. In other words, a better performance is obtained by treating $s$ as a tuning parameter.

\begin{figure}[!ht]
\centering
\subfigure[cadata]{
  \includegraphics[width=0.32\linewidth]{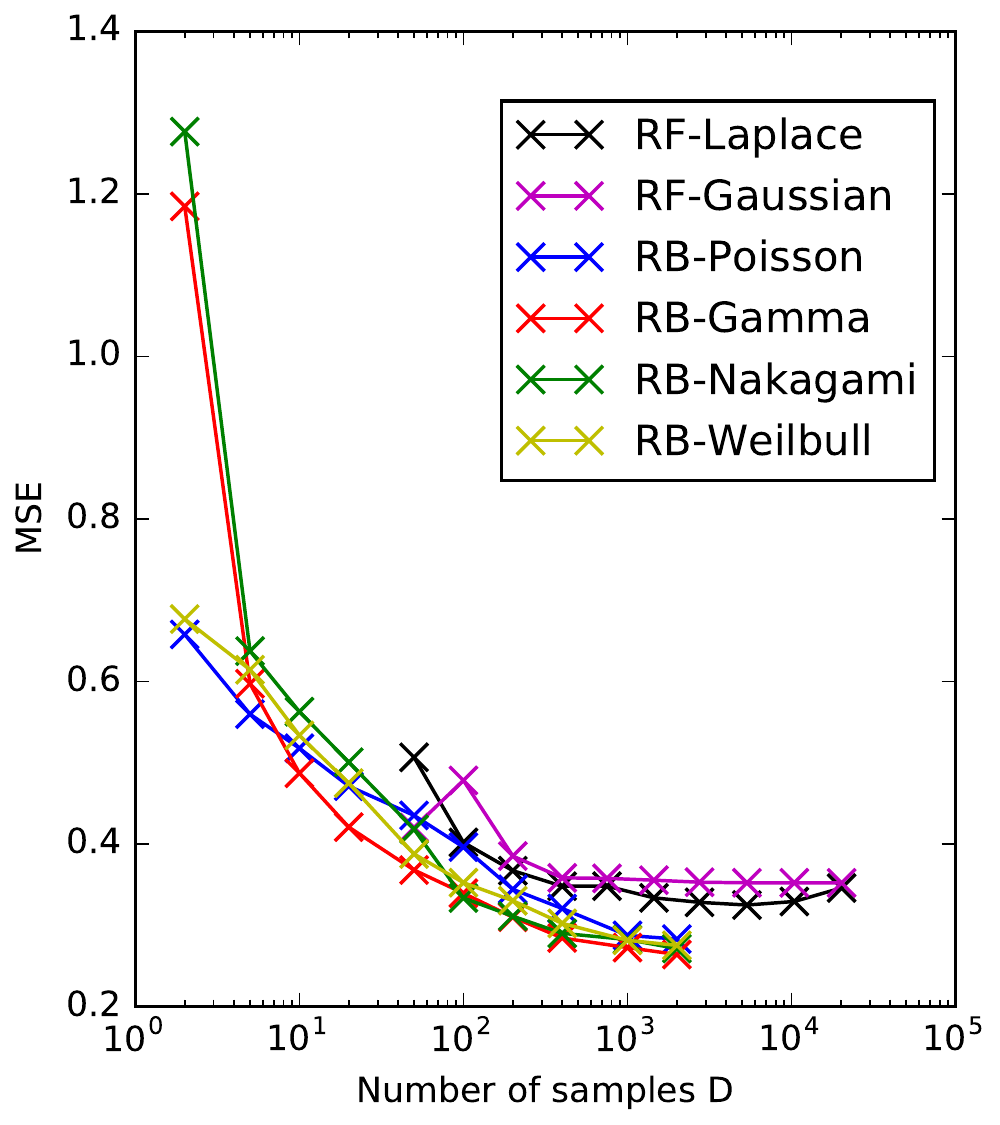}}
\subfigure[YearPredictionMSD]{
  \includegraphics[width=0.32\linewidth]{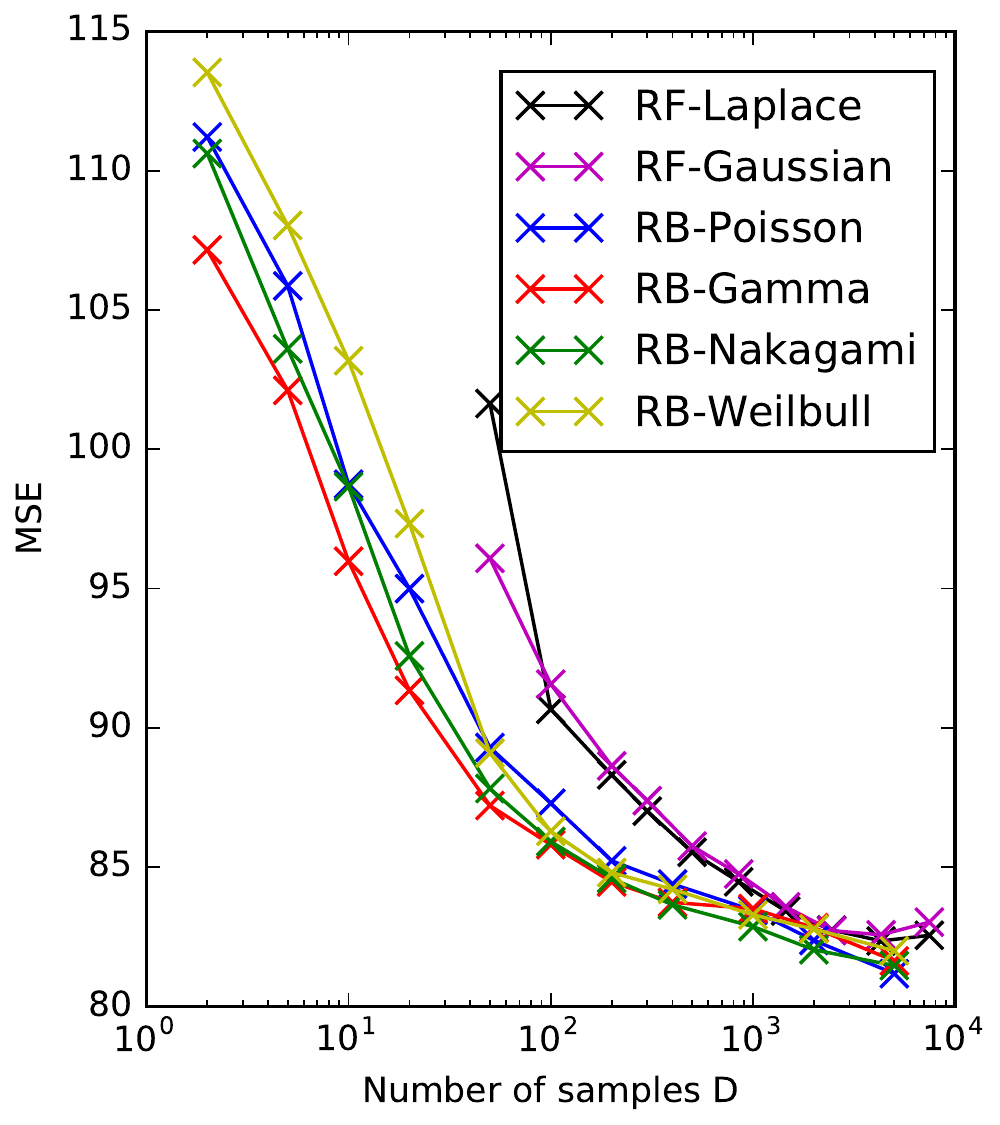}}\\
\subfigure[ijcnn1]{
  \includegraphics[width=0.32\linewidth]{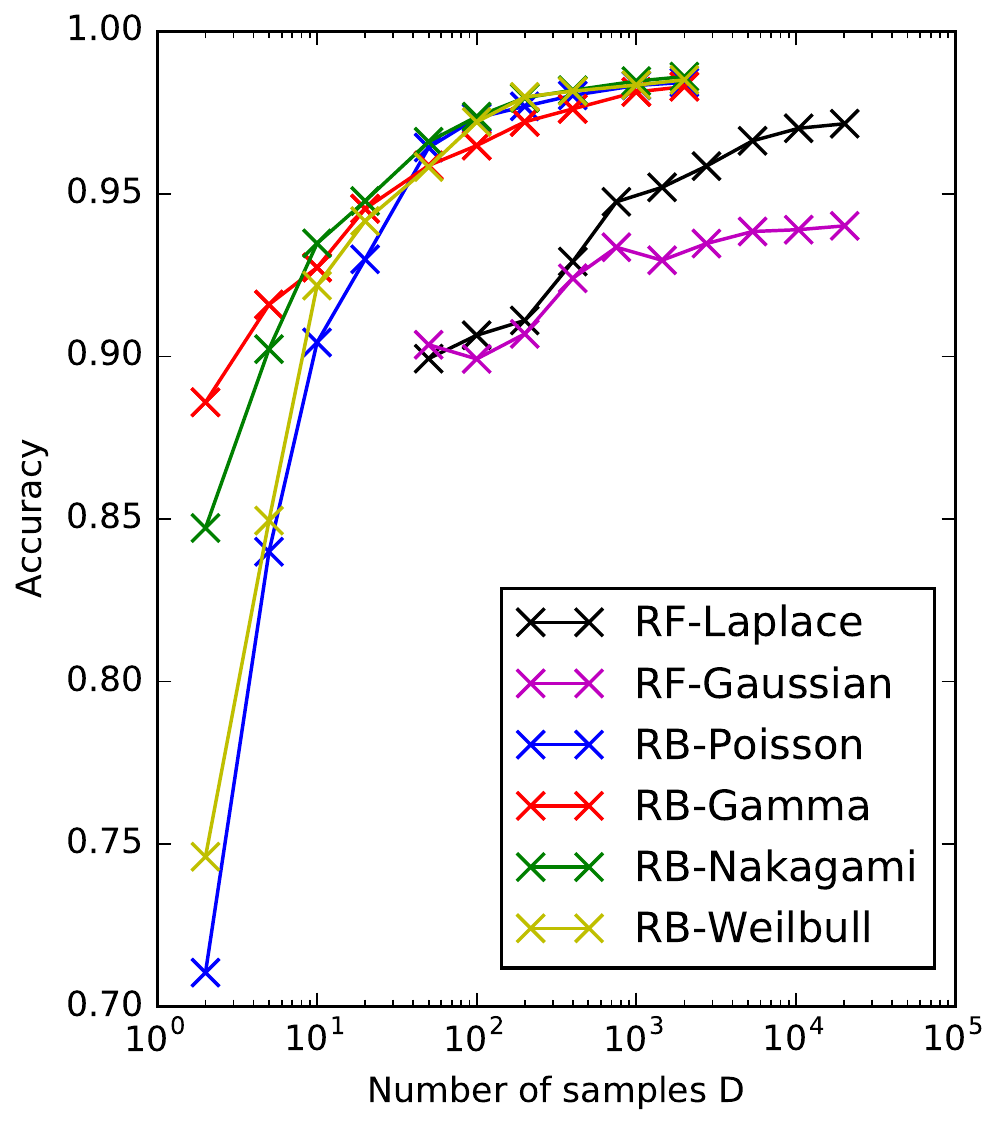}}
\subfigure[covtype.binary]{
  \includegraphics[width=0.32\linewidth]{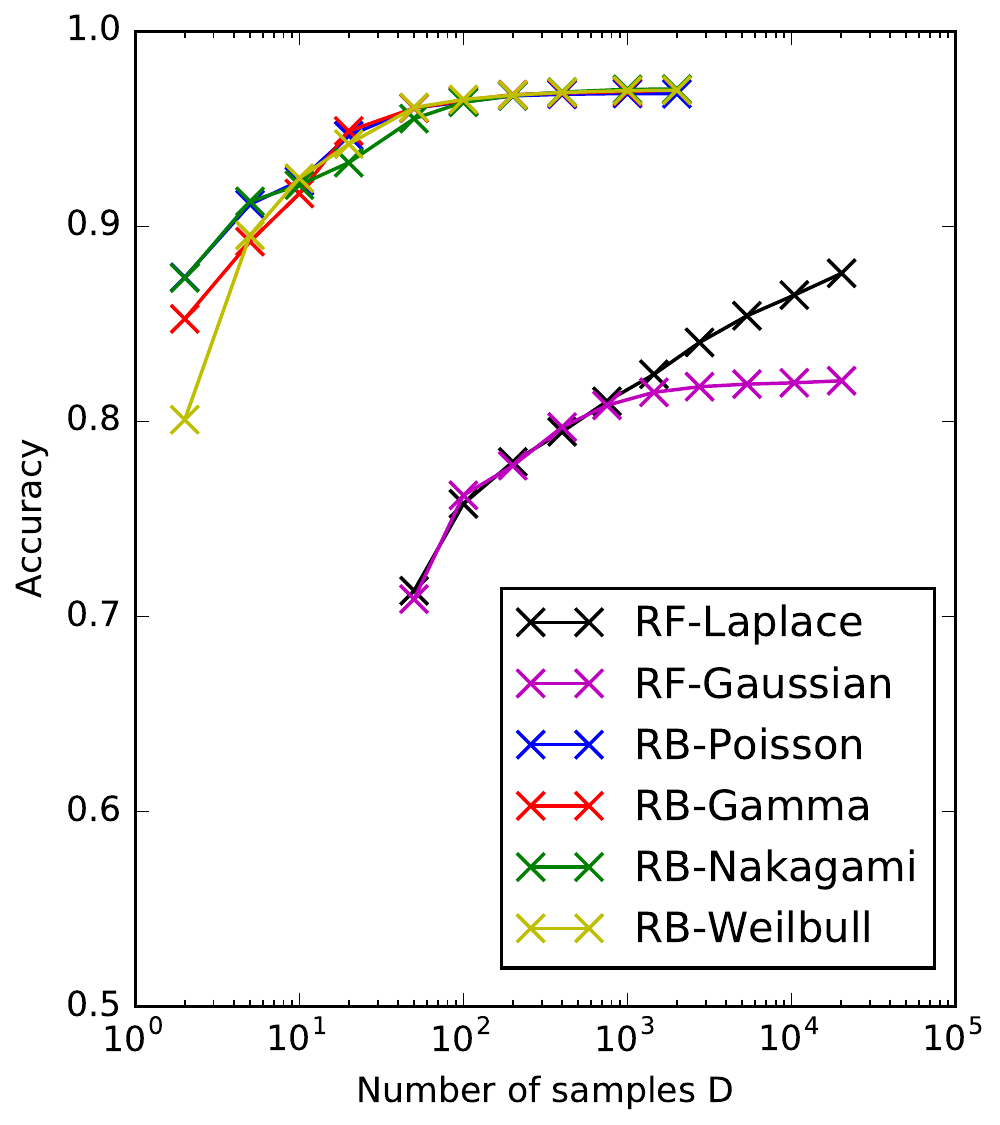}}
\subfigure[SUSY]{
  \includegraphics[width=0.32\linewidth]{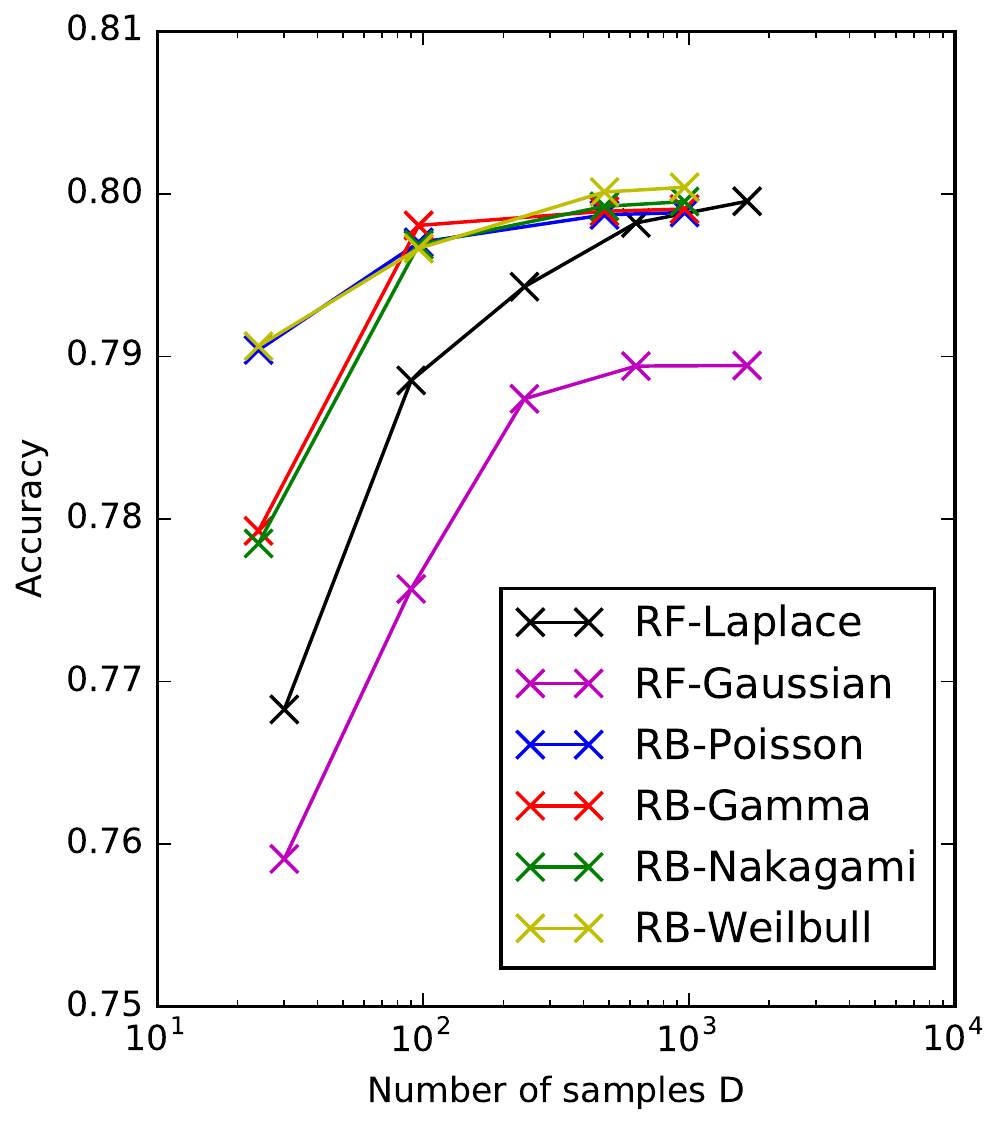}}\\
\subfigure[mnist]{
  \includegraphics[width=0.32\linewidth]{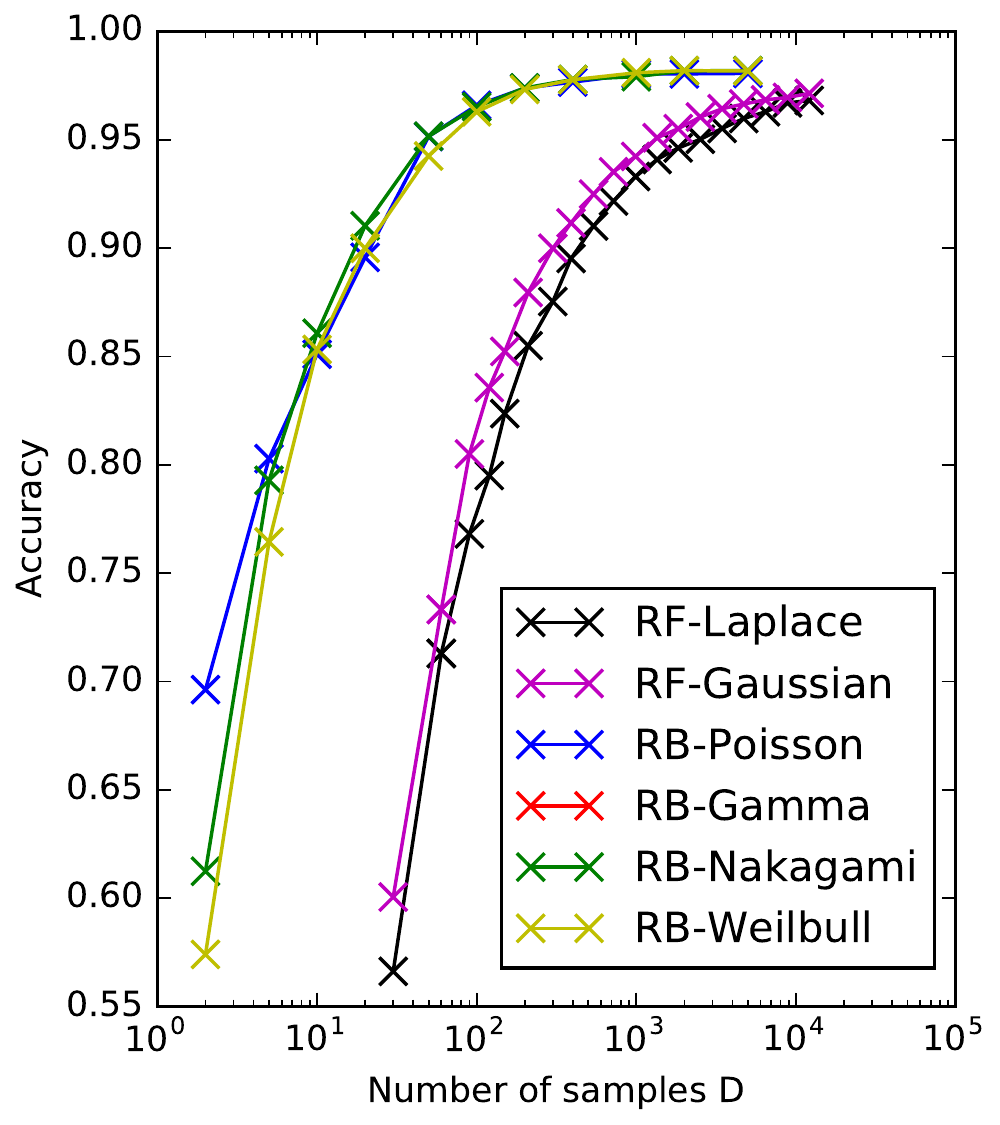}}
\subfigure[acoustic]{
  \includegraphics[width=0.32\linewidth]{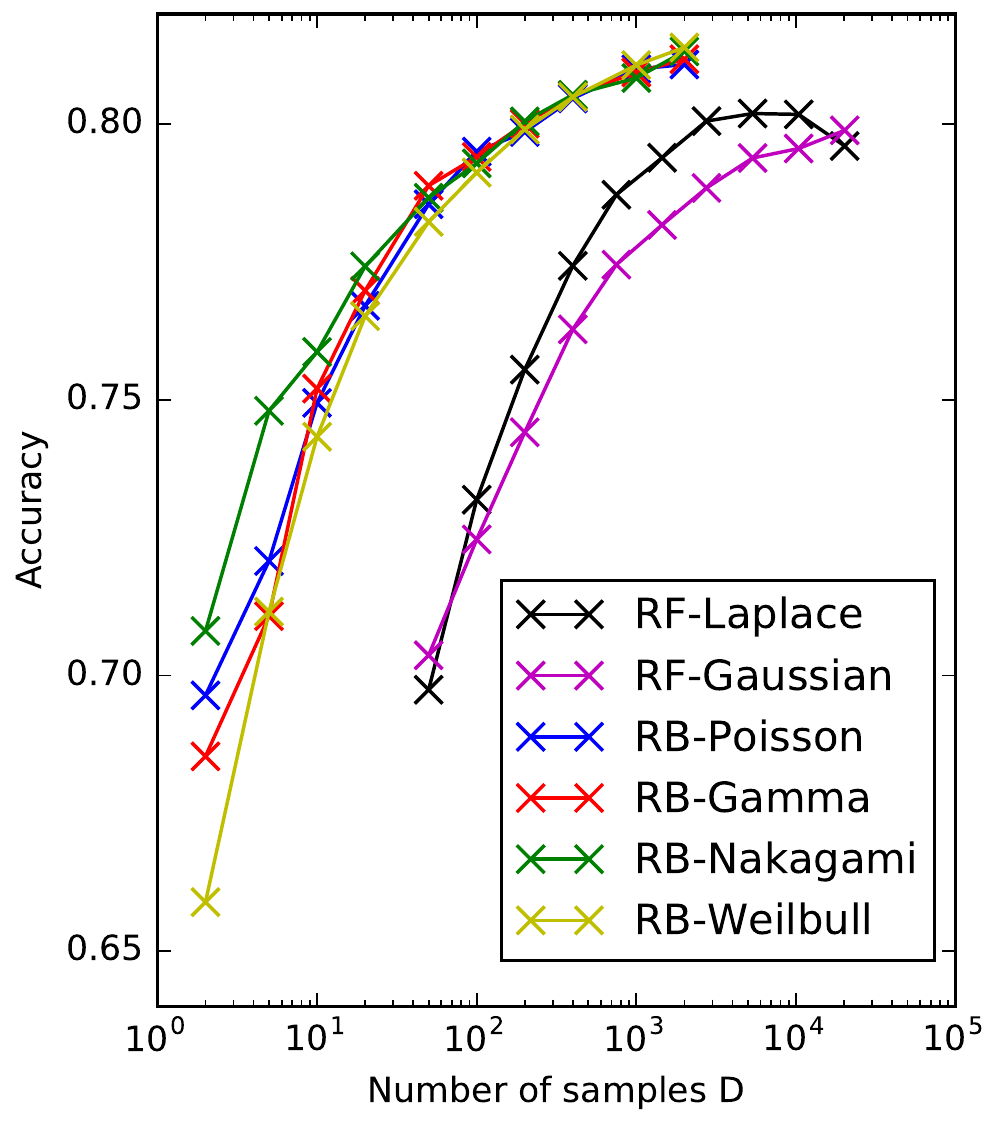}}
\subfigure[covtype]{
  \includegraphics[width=0.32\linewidth]{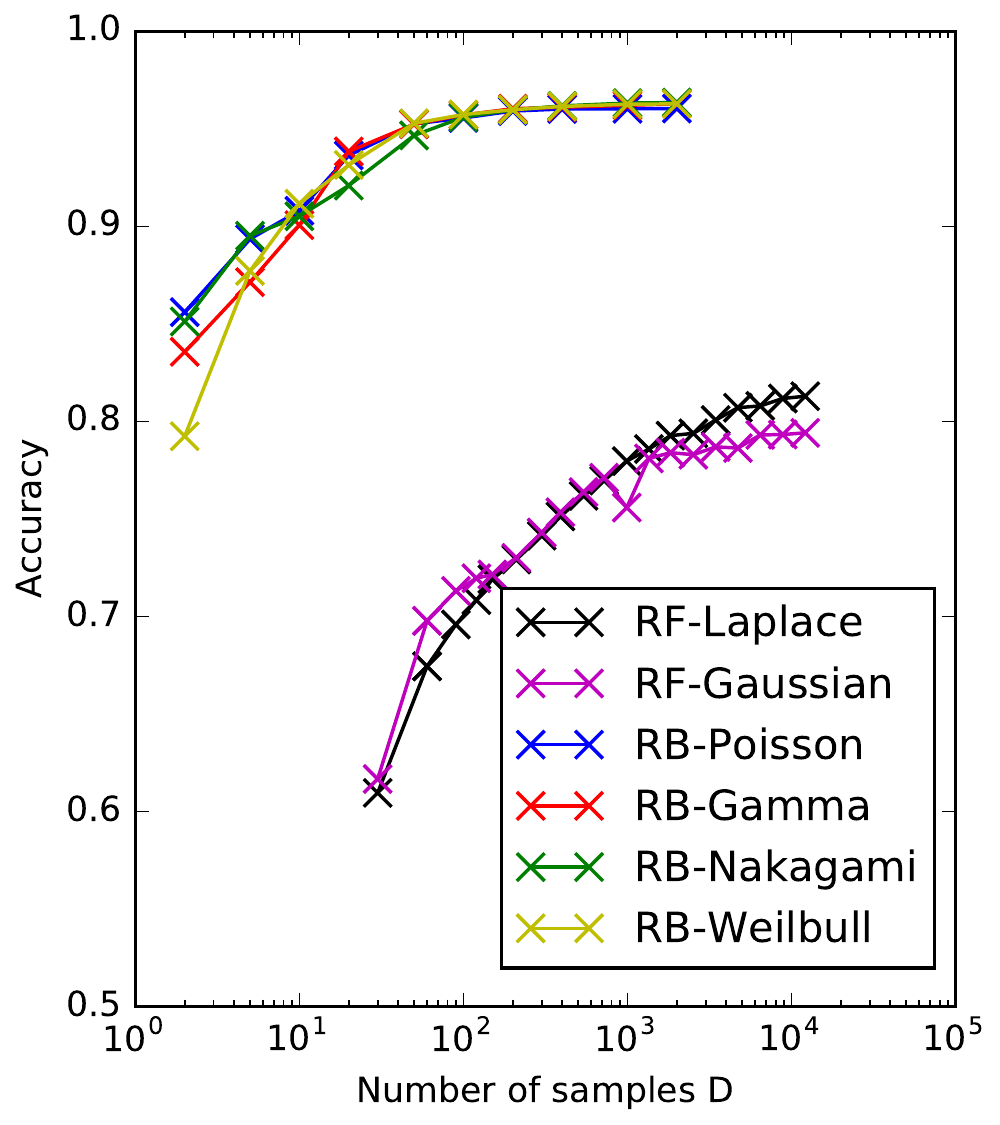}}
\caption{Regression/Classification performance as a function of sample size $D$. Top row: regression. Middle row: binary classification. Bottom row: multiclass classification.}
\label{fig:regres.classi}
\end{figure}

\begin{table}[ht]
\centering
\caption{Tuned parameters for the random binning maps.}
\label{tab:param}
\small
\setlength{\tabcolsep}{5pt}
\begin{tabular}{lr@{ }lcc}
\\[-5pt]
\multicolumn{5}{c}{cadata}\\
\hline
Distri. & \multicolumn{2}{c}{Param.} & $\tau$ & $\lambda$\\
\hline
Poisson  & $\mu=$    & 2.0 & 0.46 & 0.1\\
Gamma    & $s=$      & 0.5 & 3.16 & 0.1\\
Nakagami & $m=$      & 0.5 & 1.66 & 0.1\\
Weilbull & $\alpha=$ & 1.0 & 0.87 & 0.1\\
\hline
\end{tabular}
\hspace{3pt}
\begin{tabular}{lr@{ }lcc}
\\[-5pt]
\multicolumn{5}{c}{YearPredictionMSD}\\
\hline
Distri. & \multicolumn{2}{c}{Param.} & $\tau$ & $\lambda$\\
\hline
Poisson  & $\mu=$    & 4.0 & 0.87 & 1\\
Gamma    & $s=$      & 2.5 & 0.87 & 1\\
Nakagami & $m=$      & 2.0 & 0.87 & 1\\
Weilbull & $\alpha=$ & 3.0 & 1.66 & 1\\
\hline
\end{tabular}
\\[10pt]
\begin{tabular}{lr@{ }lcc}
\multicolumn{5}{c}{ijcnn1}\\
\hline
Distri. & \multicolumn{2}{c}{Param.} & $\tau$ & $\lambda$\\
\hline
Poisson  & $\mu=$    & 1.0 & 0.87 & 1\\
Gamma    & $s=$      & 2.0 & 0.87 & 1\\
Nakagami & $m=$      & 2.0 & 0.46 & 1\\
Weilbull & $\alpha=$ & 3.0 & 0.87 & 1\\
\hline
\end{tabular}
\hspace{3pt}
\begin{tabular}{lr@{ }lcc}
\multicolumn{5}{c}{covtype.binary}\\
\hline
Distri. & \multicolumn{2}{c}{Param.} & $\tau$ & $\lambda$\\
\hline
Poisson  & $\mu=$    & 4.0 & 0.12 & 0.1\\
Gamma    & $s=$      & 2.0 & 0.12 & 0.1\\
Nakagami & $m=$      & 1.0 & 0.24 & 0.1\\
Weilbull & $\alpha=$ & 2.0 & 0.24 & 0.01\\
\hline
\end{tabular}
\hspace{3pt}
\begin{tabular}{lr@{ }lcc}
\multicolumn{5}{c}{SUSY}\\
\hline
Distri. & \multicolumn{2}{c}{Param.} & $\tau$ & $\lambda$\\
\hline
Poisson  & $\mu=$    & 1.0 & 1.33 & 1\\
Gamma    & $s=$      & 1.5 & 1.77 & 1\\
Nakagami & $m=$      & 1.5 & 1.00 & 1\\
Weilbull & $\alpha=$ & 1.0 & 5.62 & 1\\
\hline
\end{tabular}
\\[10pt]
\begin{tabular}{lr@{ }lcc}
\multicolumn{5}{c}{mnist}\\
\hline
Distri. & \multicolumn{2}{c}{Param.} & $\tau$ & $\lambda$\\
\hline
Poisson  & $\mu=$    & 4.0 & 21.5 & 0.1\\
Gamma    & $s=$      & 2.0 & 21.5 & 0.01\\
Nakagami & $m=$      & 1.5 & 21.5 & 0.01\\
Weilbull & $\alpha=$ & 2.0 & 40.8 & 0.01\\
\hline
\end{tabular}
\hspace{3pt}
\begin{tabular}{lr@{ }lcc}
\multicolumn{5}{c}{acoustic}\\
\hline
Distri. & \multicolumn{2}{c}{Param.} & $\tau$ & $\lambda$\\
\hline
Poisson  & $\mu=$    & 0.5 & 1.66 & 1\\
Gamma    & $s=$      & 1.5 & 1.66 & 1\\
Nakagami & $m=$      & 1.5 & 0.87 & 1\\
Weilbull & $\alpha=$ & 1.0 & 5.99 & 1\\
\hline
\end{tabular}
\hspace{3pt}
\begin{tabular}{lr@{ }lcc}
\multicolumn{5}{c}{covtype}\\
\hline
Distri. & \multicolumn{2}{c}{Param.} & $\tau$ & $\lambda$\\
\hline
Poisson  & $\mu=$    & 4.0 & 0.12 & 0.1\\
Gamma    & $s=$      & 2.0 & 0.12 & 0.1\\
Nakagami & $m=$      & 1.0 & 0.24 & 0.01\\
Weilbull & $\alpha=$ & 2.0 & 0.24 & 0.01\\
\hline
\end{tabular}

\end{table}

Note that we particularly include the Gaussian kernel for comparison. Unlike other kernels, this kernel does not fall within Polya's characterization, because it is not convex on $[0,\infty)$. However, ones sees that its performance is often similar to that of Laplace. In the context of random feature maps, they are not as good as the kernels approximated by random binning.

\subsection{Matrix Approximation Error v.s. Prediction Performance}
\label{sec:approx.vs.predict}
A link is missing between the kernel matrix approximation error and a machine learning task performance. We have shown that random binning yields a better approximation from the matrix angle, and we have also demonstrated that it yields a better prediction performance from the regression/classification angle. The purpose of the final experiment is to show that these two metrics appear to be closely related.

\begin{figure}[ht]
\centering
\subfigure[cadata]{
  \includegraphics[width=0.32\linewidth]{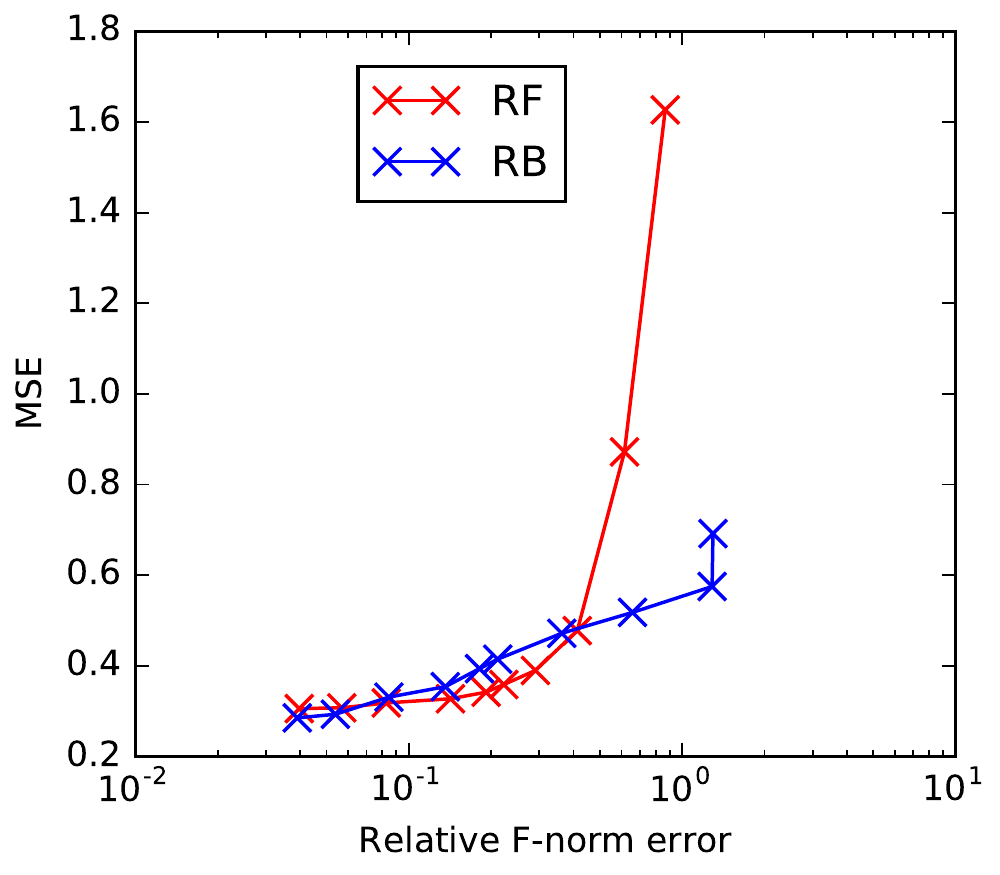}}
\subfigure[ijcnn1]{
  \includegraphics[width=0.32\linewidth]{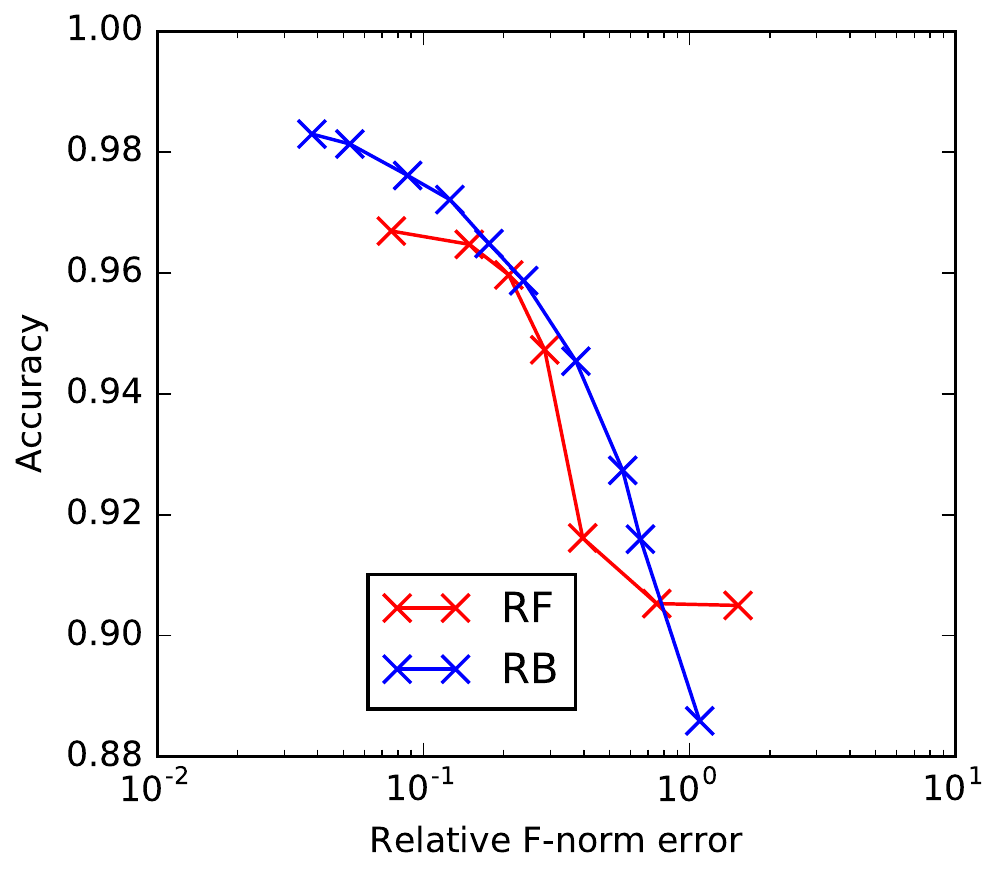}}
\subfigure[acoustic]{
  \includegraphics[width=0.32\linewidth]{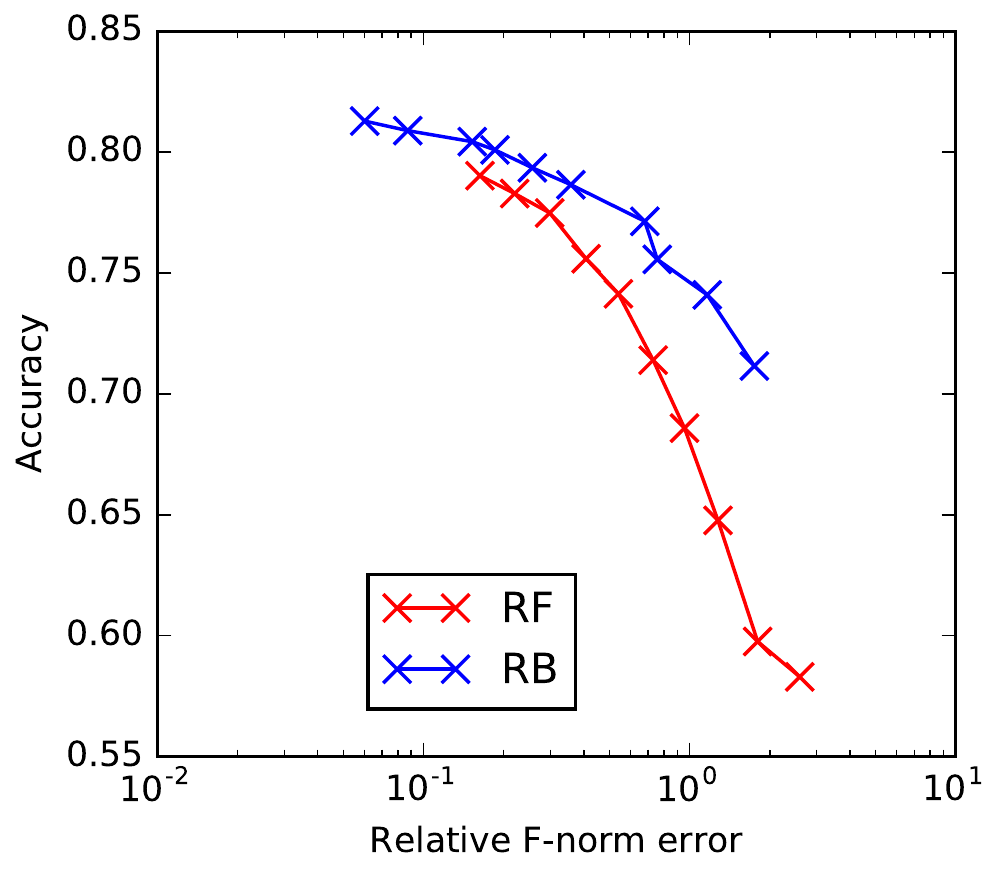}}
\caption{Matrix approximation error v.s. regression/classification performance. Laplace kernel. The two curves correspond to two methods for performing approximation.}
\label{fig:approx.vs.predict}
\end{figure}

For demonstration, we use the same data sets as in Section~\ref{sec:mat.approx}, but perform the comparison with tuned parameters obtained from the preceding subsection. In Figure~\ref{fig:approx.vs.predict} we plot the approximation error versus prediction performance. These curves are obtained for the same Laplace kernel approximated by different approaches. One sees a clear trend that a better kernel approximation implies a better prediction. Moreover, the curves of the two random feature maps are generally well aligned, indicating that the approximation method does not play a significant role in the relation between approximation error and prediction performance.

Despite the appealing empirical evidence that approximation and prediction performance are positively correlated, we, however, hesitate the conclude firmly the relation. The reader may notice in Figure~\ref{fig:regres.classi} that occasionally the prediction performance degrades when $D$ becomes too large. These scenarios occur at a large $n$, or a large $D$, that prevents us from extending the plots in Figure~\ref{fig:approx.vs.predict} for a more complete account. Incidentally, other work also shows that using the approximate kernel $\tilde{k}$ from random Fourier maps, it could happen that the prediction results are better compared with those of the nonapproximate kernel~\cite{Chen2016a}. Such phenomena appear to be beyond explanations of existing theory on the convergence of random feature maps or on the bounds of generalization error. Further theory is yet to be developed.

\section{Summary of Contributions and Conclusion}\label{sec:conclude}
This work aims at deepening the understanding of positive-definite functions, as well as the random feature maps proposed by Rahimi and Recht~\cite{Rahimi2007} for training large-scale kernel machines. We highlight a few contributions in the following.

First, we reveal that the random binning feature map is closely tied to Polya's criterion, a less used characterization of kernels compared to that of Bochner's. We derive a number of novel kernel functions~\eqref{eqn:k.pois}, \eqref{eqn:k.gamma.s}, \eqref{eqn:k.gamma.1}, \eqref{eqn:k.chi.nu}, \eqref{eqn:k.half.normal}, \eqref{eqn:k.rayleigh}, \eqref{eqn:k.nakagami.m}, \eqref{eqn:k.nakagami.1}, and \eqref{eqn:k.weibull.alpha} based on Polya's characterization, which substantially enrich the catalog of kernels applicable to kernel methods and Gaussian processes. The work~\cite{Rahimi2007} focuses on the generation of random feature maps given a kernel; hence, the sampling distributions are restricted to those tied to known kernels. On the other hand, we exploit the relationship between kernels and distributions on the opposite direction; and show that any distribution with a positive support corresponds to a valid kernel (Corollary~\ref{cor:polya}), which allows for the construction of new kernels through applying numerous known probability distributions. Additionally, we study a few properties of the kernels constructed from Polya's characterization (Theorems~\ref{thm:special.case} and~\ref{thm:area.under.curve}) and derive the Fourier transforms of the constructed kernels mentioned earlier.

Second, we compare the two approaches for generating random feature maps---random Fourier and random binning---through an analysis of the Frobenius norm error of the approximate kernel matrix (Theorems~\ref{thm:rff} and~\ref{thm:rbf}). The analysis points to a conclusion that random binning yields a smaller error in expectation. The difference in errors is demonstrated in Figure~\ref{fig:approx.err} for a few data sets. This analysis favors the random binning approach from the kernel approximation angle. Meanwhile, empirical evidences in Section~\ref{sec:exp.perf} on regression/classification performance also lead to the same preference.

Third, the revealed fact that the sampling distribution of random binning is not limited to the gamma distribution of a particular shape, allows us to treat the shape as a tuning parameter for obtaining better regression/classification performance. Moreover, it also allows us to use other distributions for chasing the performance. Figure~\ref{fig:regres.classi} and Table~\ref{tab:param} confirm this argument.

\section*{Acknowledgment}
We would like to thank Michael Stein and Haim Avron for helpful discussions. J. Chen is supported in part by the XDATA program of the Defense Advanced Research Projects Agency (DARPA), administered through Air Force Research Laboratory contract FA8750-12-C-0323. 
D. Cheng and Y. Liu are supported in part by the NSF Research Grant IIS-1254206 and IIS-1134990. The views and conclusions are those of the authors and should not be interpreted as representing the official policies of the funding agency, or the U.S. Government.
Part of the work was done while D. Cheng was a summer intern at IBM Research.

\appendix
\section{Special Functions Seen in Section~\ref{sec:dist}}
\label{app:special.function}
Gamma function
\[
\Gamma(s)=\int_0^{\infty}x^{s-1}e^{-x}\,dx,
\qquad \Re(s)>0.
\]
Upper incomplete gamma function
\[
\Gamma(s,t)=\int_t^{\infty}x^{s-1}e^{-x}\,dx,
\qquad \Re(s)\ge0.
\]
Exponential integral
\[
E_1(z)=\int_z^{\infty}\frac{e^{-t}}{t}\,dt,
\qquad |\arg(z)|<\pi.
\]
Kummer's confluent hypergeometric function
\[
M(a,b,z)=\sum_{n=0}^{\infty}\frac{a^{(n)}z^n}{b^{(n)}n!},
\qquad\text{where } a^{(0)}=1,\,\, a^{(n)}=a(a+1)(a+2)\cdots(a+n-1).
\]
Error function
\[
\erf(x)=\frac{1}{\sqrt{\pi}}\int_{-x}^x e^{-t^2}\,dt.
\]
Imaginary error function
\[
\erfi(x)=-\im\erf(\im x).
\]
Complementary error function
\[
\erfc(x)=1-\erf(x).
\]

\bibliographystyle{plain}
\bibliography{reference}

\begin{figure}[ht]
\centering
\subfigure[$X\sim$ shifted Pois$(\mu)$]{
  \includegraphics[width=.4\linewidth]{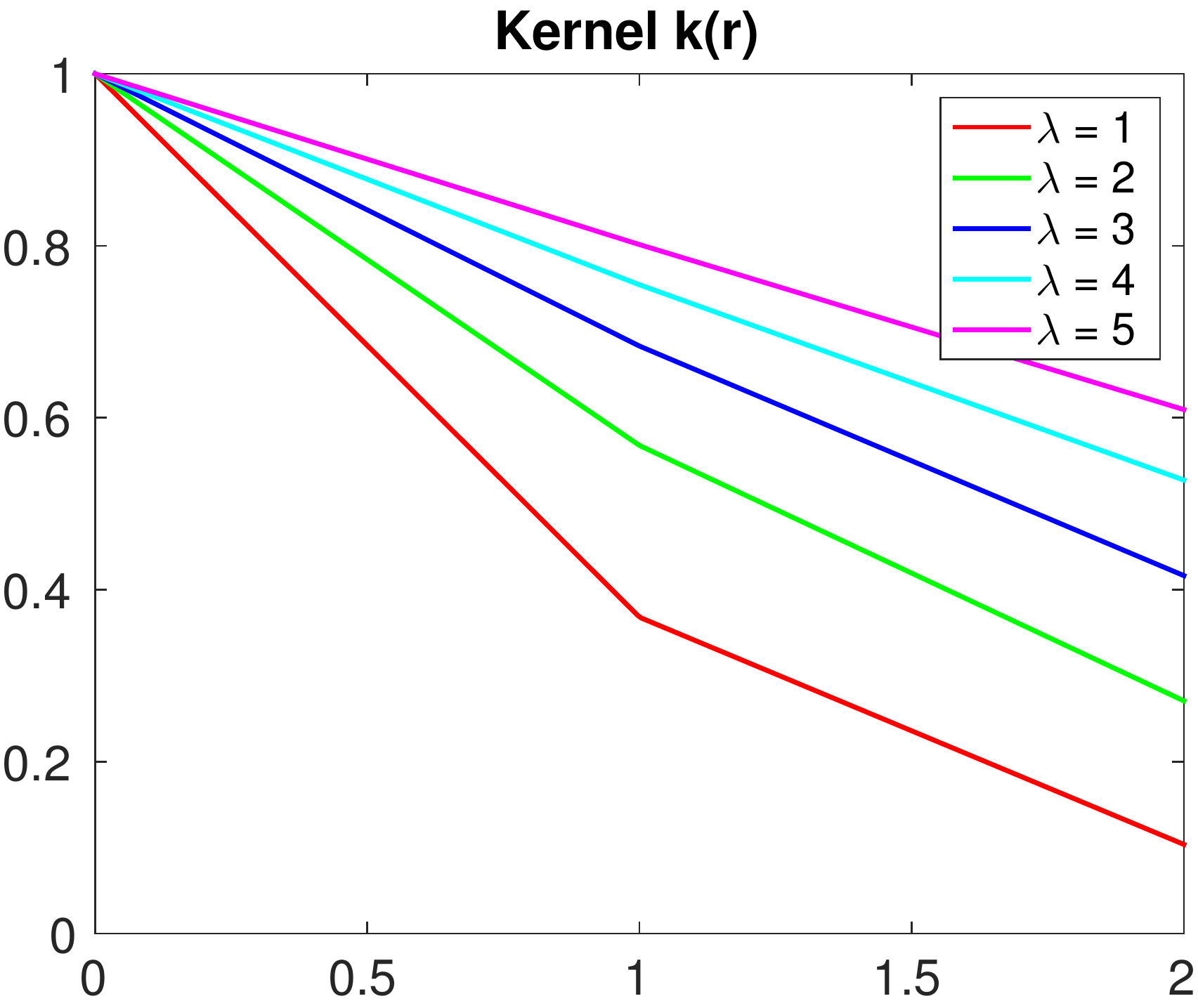}\hspace{0.1cm}
  \includegraphics[width=.4\linewidth]{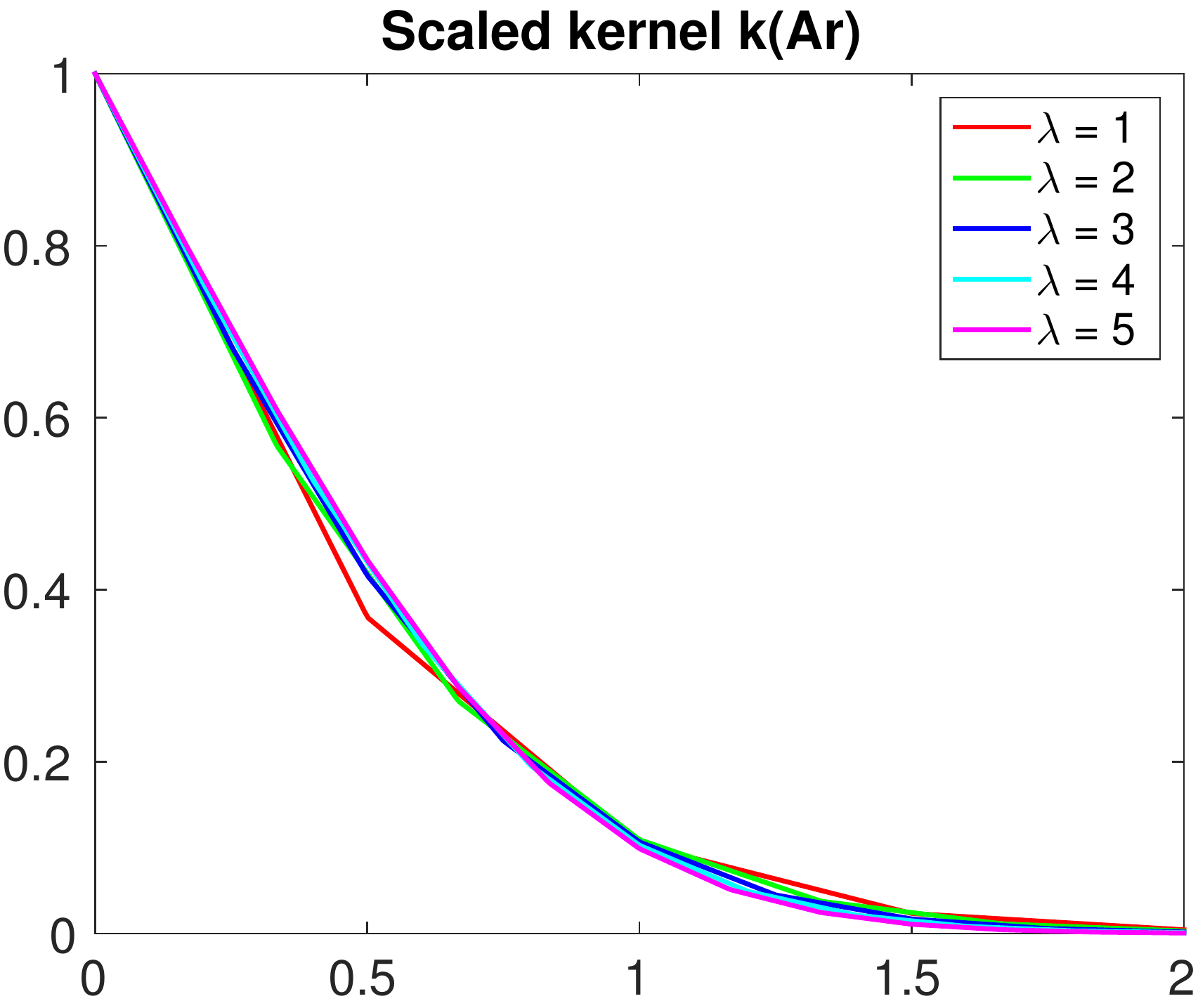}\hspace{0.1cm}}
\subfigure[$X\sim$ Gamma$(s,1)$]{
  \includegraphics[width=.4\linewidth]{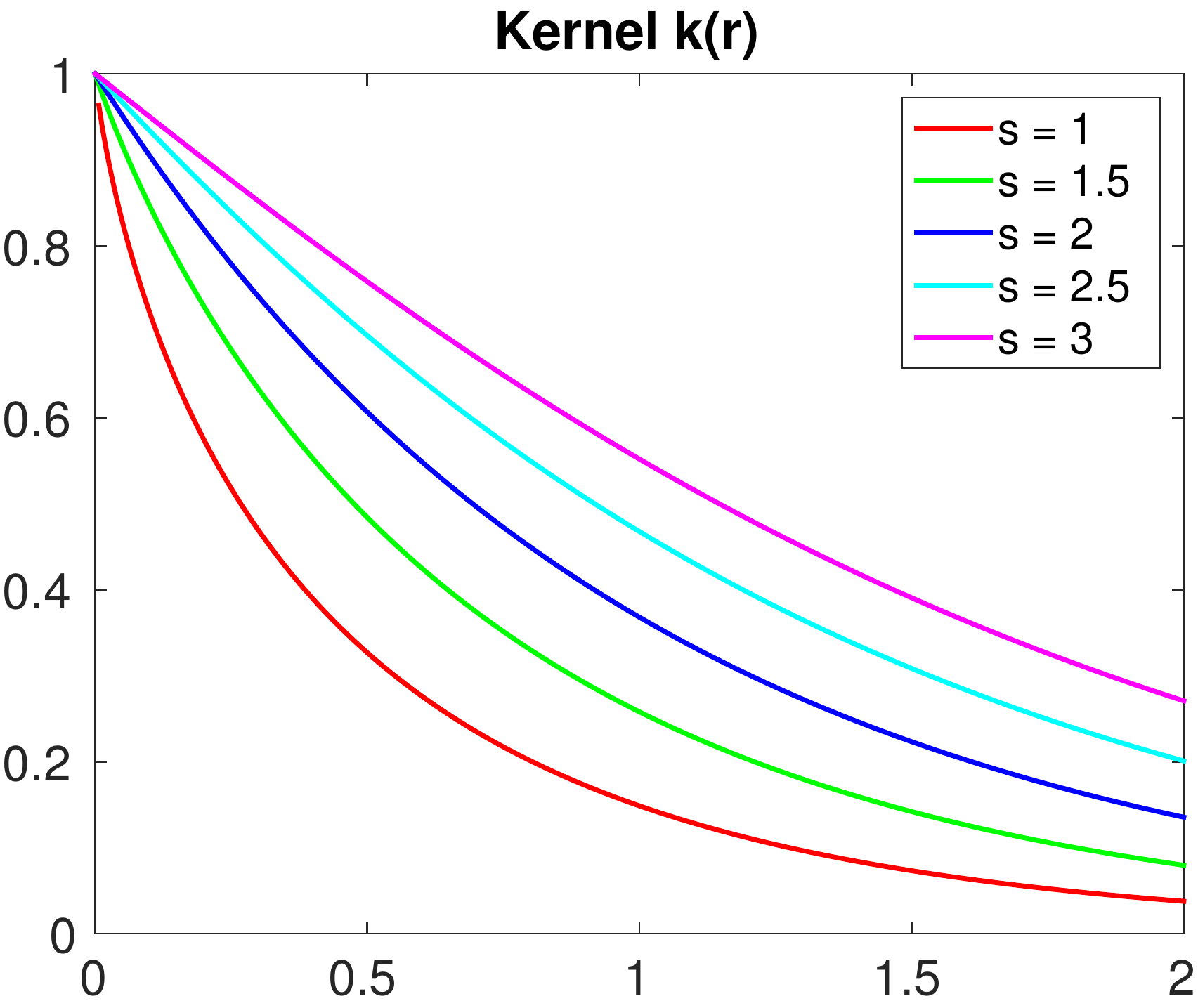}\hspace{0.1cm}
  \includegraphics[width=.4\linewidth]{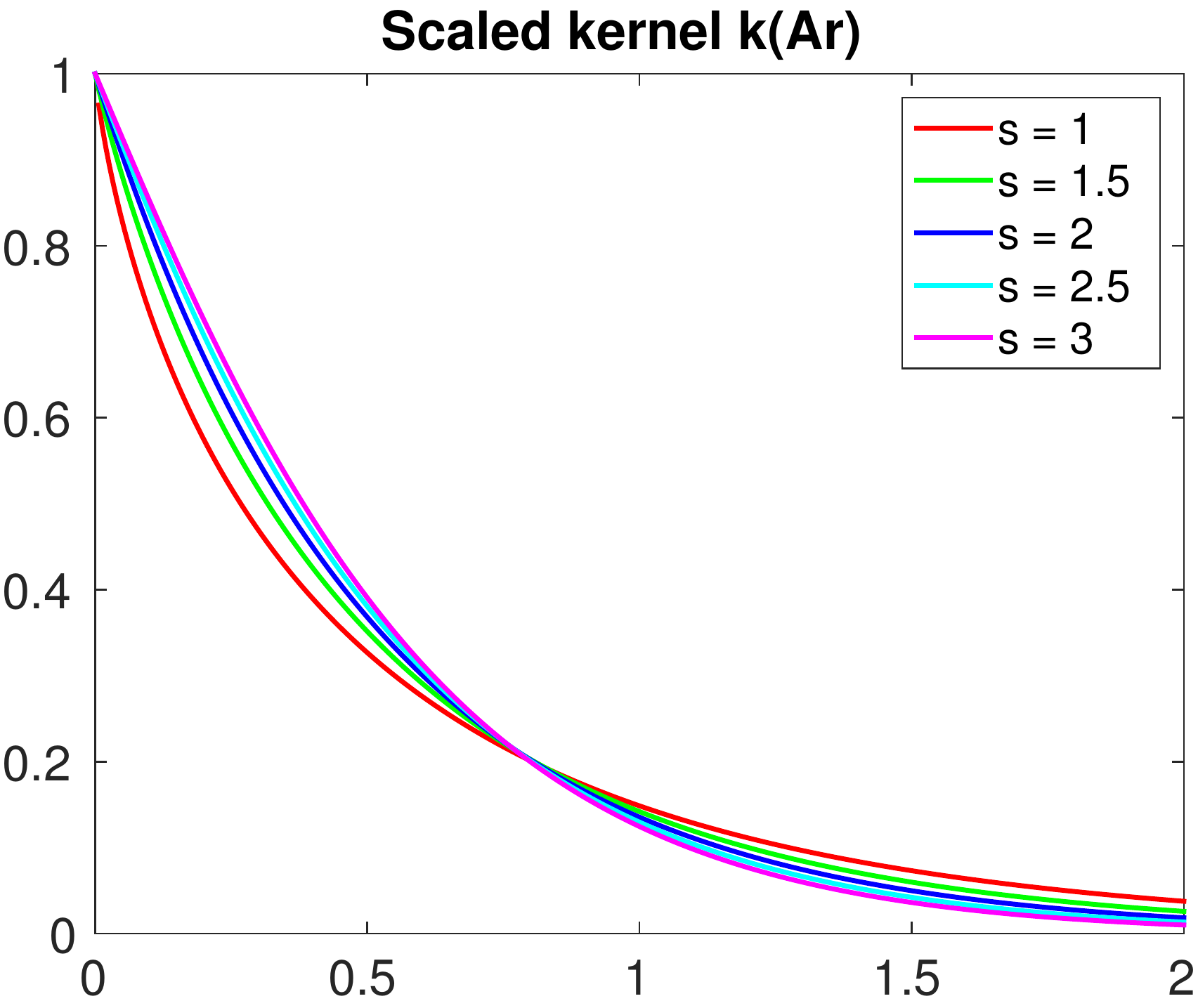}\hspace{0.1cm}}
\caption{Constructed kernels from different probability distributions. Left: Unscaled. Right: Scaled by $A=E[X]$.}
\label{fig:kernel.var}
\end{figure}

\begin{figure}[ht]
\centering
\subfigure[$X\sim$ Nakagami$(m,1)$]{
  \includegraphics[width=.4\linewidth]{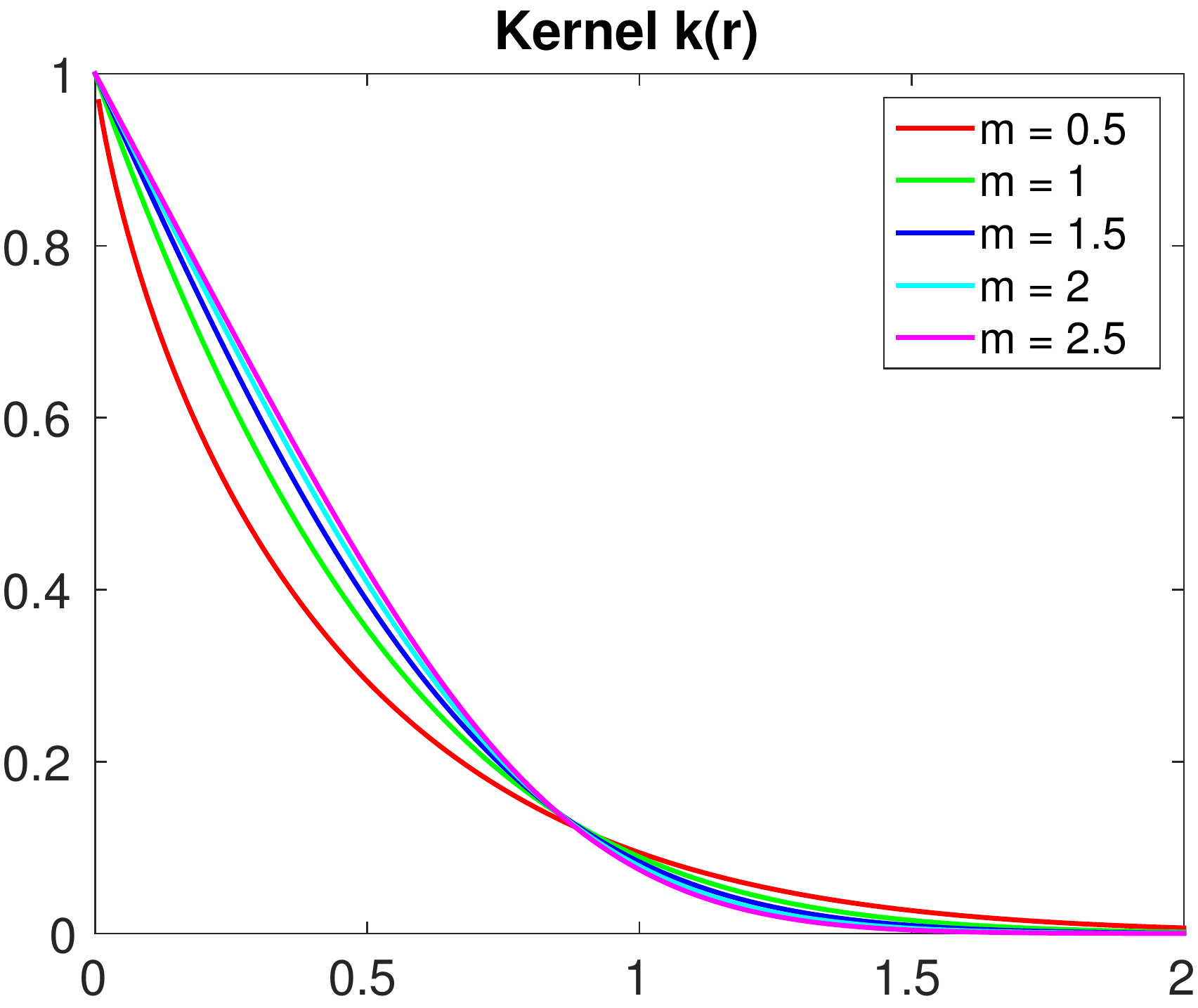}\hspace{0.1cm}
  \includegraphics[width=.4\linewidth]{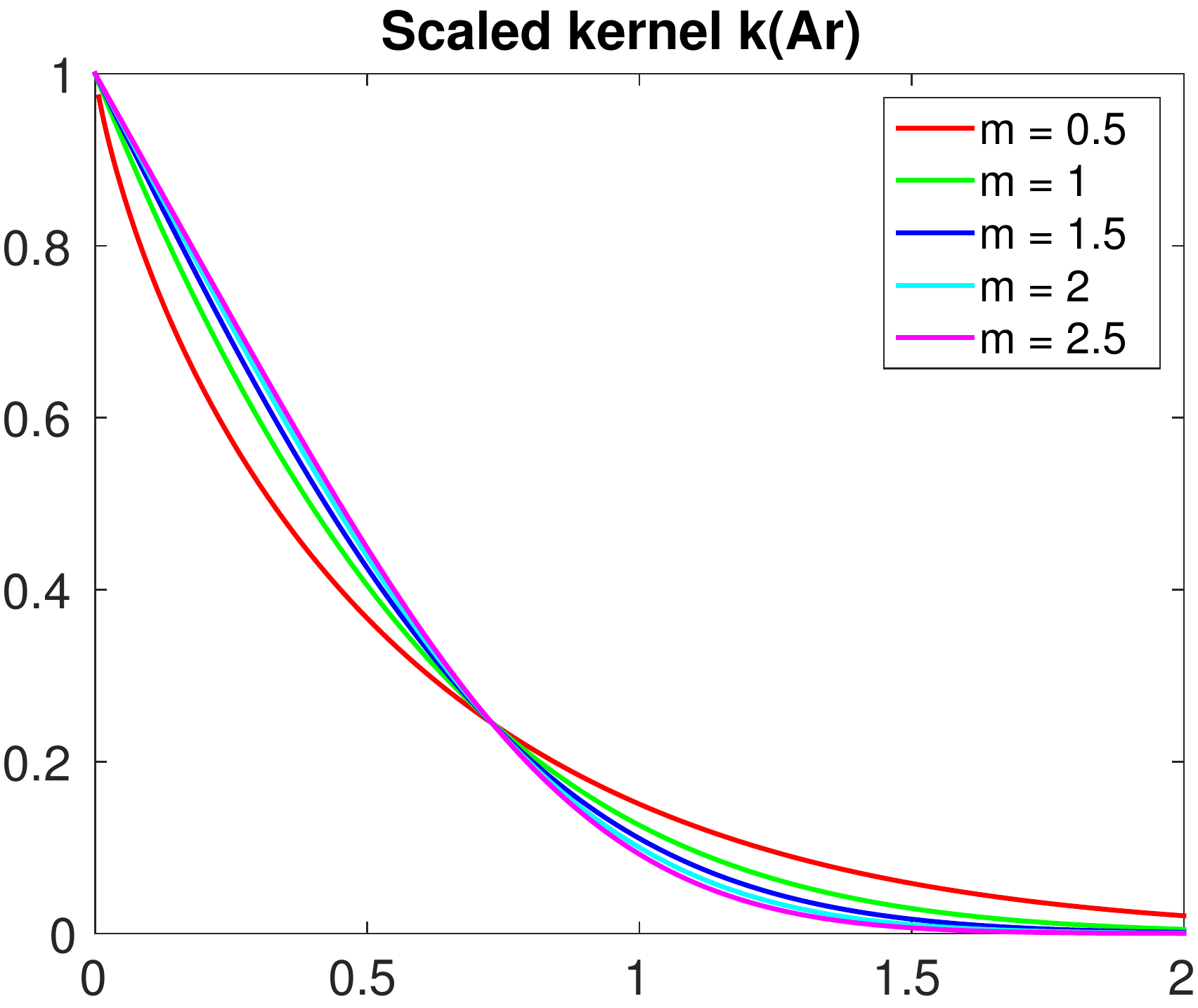}\hspace{0.1cm}}
\subfigure[$X\sim$ Weibull$(1,\alpha)$]{
  \includegraphics[width=.4\linewidth]{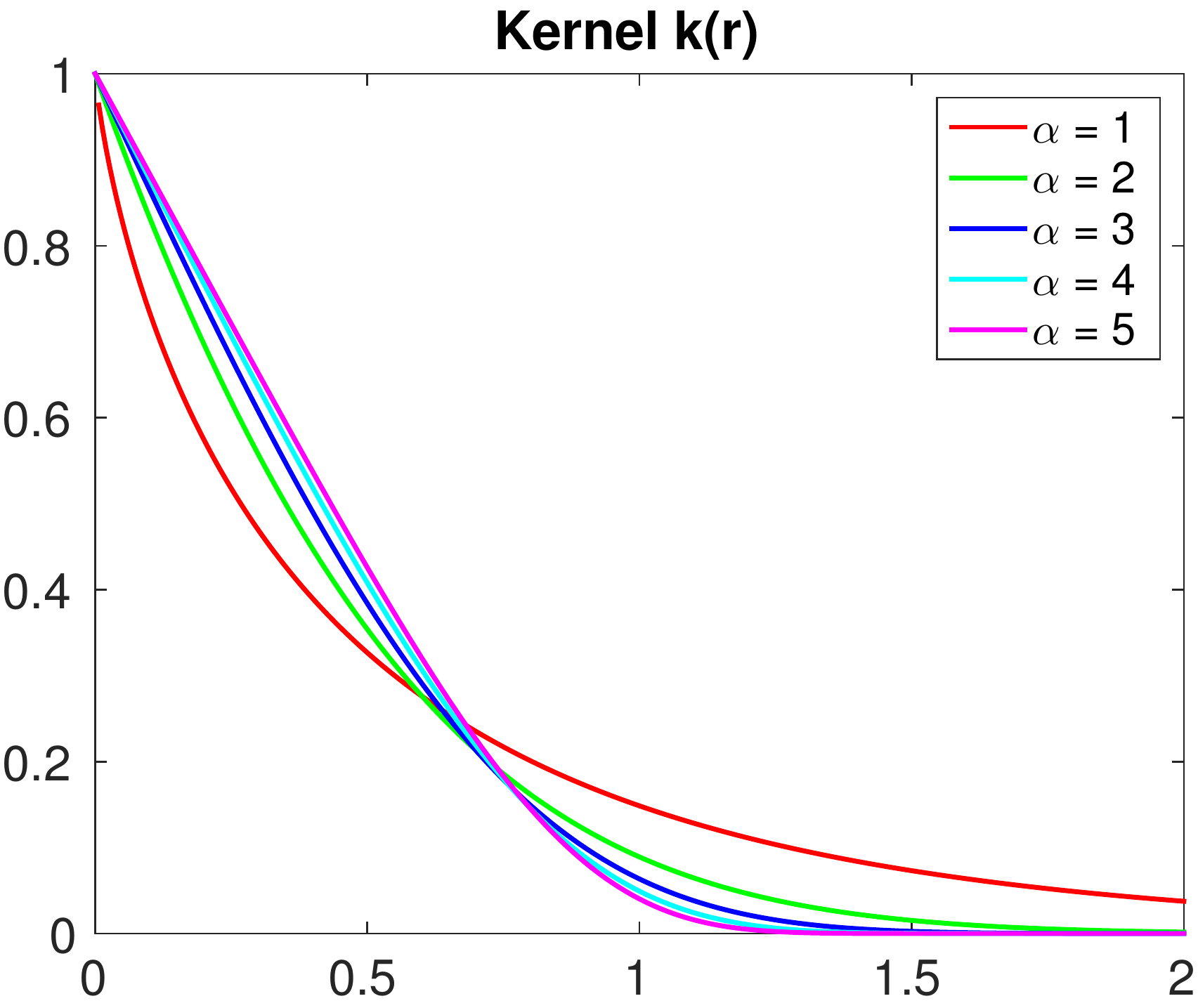}\hspace{0.1cm}
  \includegraphics[width=.4\linewidth]{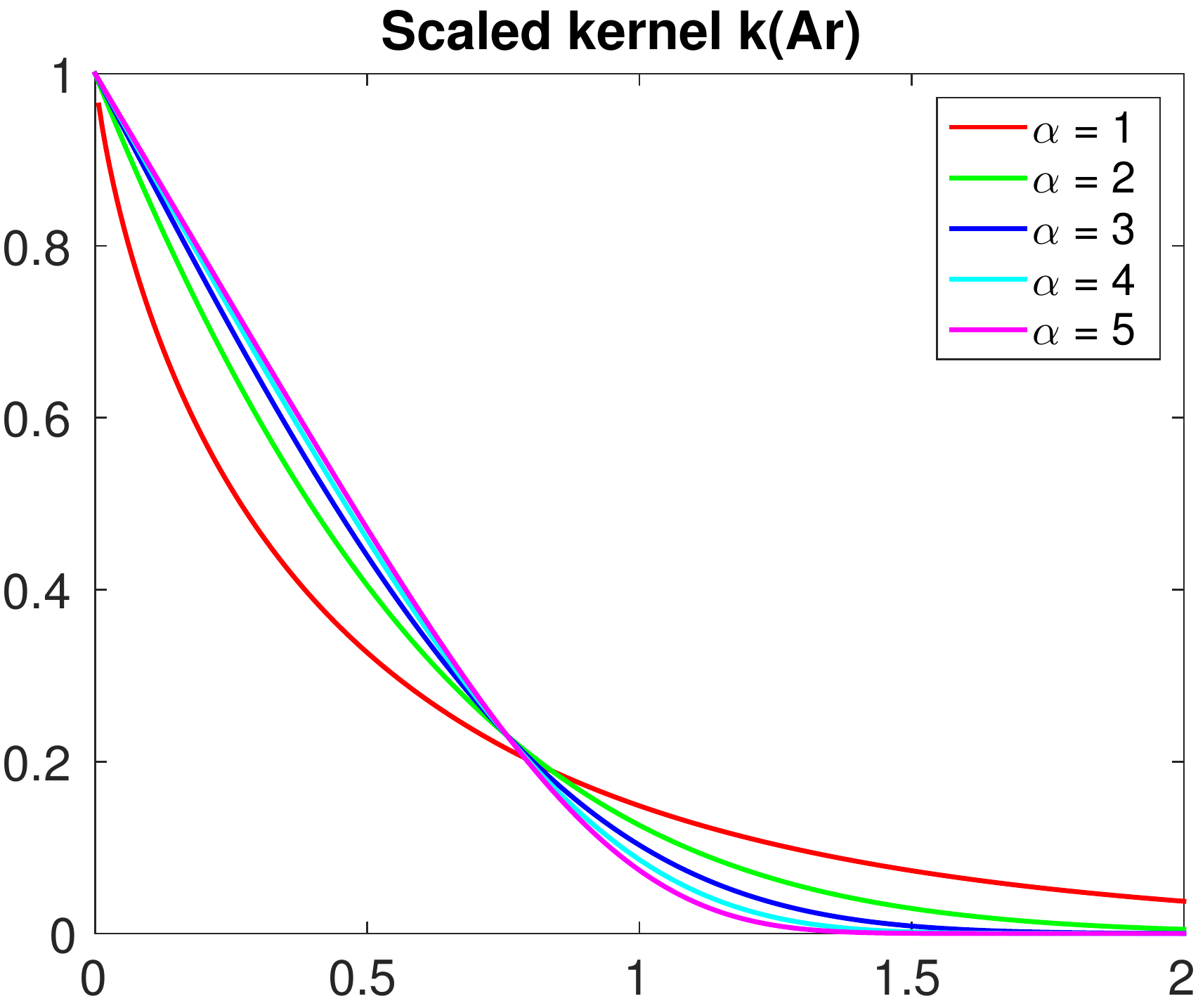}\hspace{0.1cm}}
\caption{(Continued from Figure~\ref{fig:kernel.var}) Constructed kernels from different probability distributions. Left: Unscaled. Right: Scaled by $A=E[X]$.}
\label{fig:kernel.var2}
\end{figure}

\begin{landscape}
\renewcommand{\arraystretch}{2}
\begin{table}[ht]
\centering
\caption{Probability distributions and corresponding kernels}
\label{tab:dist.kernel}
\begin{tabular}{cccccccc}
\hline
\rowcolor{gray!25}
Distribution & pmf & $E[X]$ & $k(r)$ & $\mathcal{F}[k](t)$ & Parameter choice/tuning & \\
\hline
Shifted Pois$(\mu)$ 
& $\displaystyle\frac{\mu^{x-1}e^{-\mu}}{(x-1)!}$, \,\, $x=1,2,\ldots$
& $\mu+1$
& \eqref{eqn:k.pois}
& \eqref{eqn:ft.pois}
& $\mu=1,2,3,\ldots$\\[5pt]
\hline
\rowcolor{gray!25}
Distribution & pdf & $E[X]$ & $k(r)$ & $\mathcal{F}[k](t)$ & Parameter choice/tuning & \\
\hline
Gamma$(s,\theta)$ 
& $\displaystyle\frac{x^{s-1}e^{-x/\theta}}{\Gamma(s)\theta^s}$
& $\theta s$
& \pbox[c]{\textwidth}{\relax\ifvmode\centering\fi \eqref{eqn:k.gamma.s} $s>1$\\ \eqref{eqn:k.gamma.1} $s=1$}
& \pbox[c]{\textwidth}{\relax\ifvmode\centering\fi \eqref{eqn:ft.gamma.s} $s>1$\\ \eqref{eqn:ft.gamma.1} $s=1$}
& $s=\frac{1}{2},1,\frac{3}{2},\ldots$;\quad$\theta=1$\\[5pt]
\hline
Nakagami$(m,\Omega)$ 
& $\displaystyle\frac{2m^mx^{2m-1}e^{-mx^2/\Omega}}{\Gamma(m)\Omega^m}$
& $\displaystyle\frac{\Gamma(m+\frac{1}{2})}{\Gamma(m)}\left(\frac{\Omega}{m}\right)^{1/2}$
& \pbox[c]{\textwidth}{\relax\ifvmode\centering\fi \eqref{eqn:k.nakagami.m} $m>1/2$\\ \eqref{eqn:k.nakagami.1} $m=1/2$}
& \pbox[c]{\textwidth}{\relax\ifvmode\centering\fi \eqref{eqn:ft.nakagami.m} $m>1/2$\\ \eqref{eqn:ft.nakagami.1} $m=1/2$}
& $m=\frac{1}{2},1,\frac{3}{2},\ldots$;\quad$\Omega=1$\\[5pt]
\hline
Weibull$(\theta,\alpha)$ 
& $\displaystyle\frac{\alpha}{\theta}\left(\frac{x}{\theta}\right)^{\alpha-1}e^{-(x/\theta)^{\alpha}}$
& $\theta\Gamma(1+1/\alpha)$
& \pbox[c]{\textwidth}{\relax\ifvmode\centering\fi \eqref{eqn:k.weibull.alpha} $\alpha>1$\\ \eqref{eqn:k.gamma.1} $\alpha=1$}
& \pbox[c]{\textwidth}{\relax\ifvmode\centering\fi \hspace{10pt}\phantom{$\alpha>1$}\\ \eqref{eqn:ft.gamma.1} $\alpha=1$}
& $\theta=1$;\quad$\alpha=1,2,3,\ldots$\\[5pt]
\hline
\rowcolor{gray!25}
Distribution & pdf & $E[X]$ & $k(r)$ & $\mathcal{F}[k](t)$ & Same as & \\
\hline
Exp$(\theta)$ 
& $\displaystyle\frac{1}{\theta}e^{-x/\theta}$
& $\theta$
& \eqref{eqn:k.gamma.1}
& \eqref{eqn:ft.gamma.1}
& \pbox[c]{\textwidth}{\relax\ifvmode\centering\fi Gamma$(1,\theta)$\\ Weibull$(\theta,1)$}\\[5pt]
\hline
$\chi_{\nu}^2$ 
& $\displaystyle\frac{x^{\nu/2-1}e^{-x/2}}{2^{\nu/2}\Gamma(\nu/2)}$
& $\nu$
& \eqref{eqn:k.gamma.s}
& \eqref{eqn:ft.gamma.s}
& Gamma$(\nu/2,2)$\\[5pt]
\hline
$\chi_{\nu}$ 
& $\displaystyle\frac{2^{1-\nu/2}}{\Gamma(\nu/2)}x^{\nu-1}e^{-x^2/2}$
& $\displaystyle\sqrt{2}\frac{\Gamma((\nu+1)/2)}{\Gamma(\nu/2)}$
& \pbox[c]{\textwidth}{\relax\ifvmode\centering\fi \eqref{eqn:k.chi.nu} $\nu>1$\\ \eqref{eqn:k.half.normal} $\nu=1$}
& \pbox[c]{\textwidth}{\relax\ifvmode\centering\fi \eqref{eqn:ft.chi.nu} $\nu>1$\\ \eqref{eqn:ft.half.normal} $\nu=1$}
& Nakagami$(\nu/2, \nu)$\\[5pt]
\hline
HN$(\sigma)$ 
& $\displaystyle\frac{\sqrt{2}}{\sigma\sqrt{\pi}}\exp\left(-\frac{x^2}{2\sigma^2}\right)$
& $\displaystyle\frac{\sigma\sqrt{2}}{\sqrt{\pi}}$
& \eqref{eqn:k.half.normal}
& \eqref{eqn:ft.half.normal}
& Nakagami$(1/2, \sigma^2)$\\[5pt]
\hline
Rayleigh($\sigma$) 
& $\displaystyle\frac{x}{\sigma^2}e^{-x^2/(2\sigma^2)}$
& $\displaystyle\sigma\sqrt{\frac{\pi}{2}}$
& \eqref{eqn:k.rayleigh}
& \eqref{eqn:ft.rayleigh}
& \pbox[c]{\textwidth}{\relax\ifvmode\centering\fi Nakagami$(1, 2\sigma^2)$\\ Rice$(0,\sigma)$}\\[5pt]
\hline
\end{tabular}
\end{table}
\end{landscape}

\end{document}